\documentclass[11pt]{amsart}

\usepackage{Style_files/Packages}
\usepackage{Style_files/New_commands}
\usepackage{Style_files/Style}

\SetSymbolFont{stmry}{bold}{U}{stmry}{m}{n}

\begin{document}

\title{Stochastic Differential Equations models for Least-Squares Stochastic Gradient Descent}

\author{Adrien Schertzer$^\dagger$ and Loucas Pillaud-Vivien$^{\star}$}
\address{$^\dagger$ Universität Bonn, email: aschertz@uni-bonn.de}
\address{$^\star$ Ecole des Ponts ParisTech - Cermics, email: loucas.pillaud-vivien@enpc.fr}

\maketitle
\begin{abstract}
We study the dynamics of a continuous-time model of the Stochastic Gradient Descent (SGD) for the least-square problem. Indeed, pursuing the work of \cite{li2019stochastic}, we analyze Stochastic Differential Equations (SDEs) that model SGD either in the case of the \textit{training loss} (finite samples) or the population one (online setting). A key qualitative feature of the dynamics is the existence of a perfect interpolator of the data, irrespective of the sample size. In both scenarios, we provide precise, non-asymptotic rates of convergence to the (possibly degenerate) stationary distribution. Additionally, we describe this asymptotic distribution, offering estimates of its mean, deviations from it, and a proof of the emergence of heavy-tails related to the step-size magnitude. Numerical simulations supporting our findings are also presented.
\end{abstract}

\section{Introduction}

The stochastic gradient descent (SGD) is the workhorse of any large-scale machine learning pipeline. Described more than seventy years ago as a versatile stochastic algorithm~\cite{robbins1951}, it has been studied thoroughly since then (e.g.~\cite{pages2001quelques,benaim2006dynamics}), and applied with success as an efficient computational and statistical device for large scale machine learning~\cite{bottou2008tradeoffs}. Yet, since the emergence of deep-neural networks (DNNs), some new lights have been shed on the use of the algorithm: among other, the role of SGD's noise in the good generalization performances of DNNs~\cite{zhang2021understanding} has been described empirically~\cite{xing2018walk}, while its particular shape of covariance is expected to play a predominant role is its dynamics~\cite{andriushchenko2023sgd}. 

In this direction, \cite{li2019stochastic} proposed to approximate the SGD dynamics thanks to a Stochastic Differential Equation (SDE) whose noise's covariance matches the one of SGD. Since then, many works have leveraged this continuous perspective, attempting to better describe some of SGD's phenomenology: among other, the role of step-size~\cite{haochen2021shape}, the escape time and direction from local minimizers~\cite{xie2020diffusion}, study of the invariant distribution~\cite{chaudhari2018stochastic} or the heavy tail phenomenon~\cite{gurbuzbalaban2021heavy}. However, lured by the strong analytical tools that offer SDE models, some misconceptions on the basic nature of SGD's noise \textit{in the machine learning context} have also emerged. Among other, one important message of \cite{wojtowytsch2023stochastic} is to recall that SGD's noise has a particular shape and intensity that make SGD far from a \textit{Langevin} type of dynamics where isotropic noise is added to the gradient at each step. Notably, if the data model can be interpolated perfectly, the invariant measure of SGD could be largely degenerate without need of step-size decay: this is the case in the overparametrized regime~\cite{berthier2020tight,varre2021last}. 

In this article, we take a step back from general purpose studies of SGD and focus on the specific case of least-squares, where the predictor is a linear function of fixed features. Needless to say that linear predictors do not lead to state-of-the art performance in most modern tasks, yet, the abundant literature on the \textit{neural tangent kernel}~\cite{jacot2018neural} recalled us that we have still to understand better linear models. Moreover, even if the article is not written within the setting of \textit{kernel methods} for the sake of simplicity, every assumption is made in order to easily adapt to this setting, replacing $\R^d$ equipped with its canonical Euclidean structure by an abstract reproducing kernel Hilbert space. 

\paragraph{Purpose and contributions.} The aim of the present paper is to show that, with a proper model, SDEs offer nice analytical tools that help to clean the analysis while capturing the qualitative as well as the quantitative essence of SGD's dynamics. We present a wide range of results on both online and empirical SGD, demonstrating through quantitative analysis that the key difference lies in the model's ability to achieve perfect interpolation. We also present systemically non-asymptotic rates of convergence in all the settings we present, either in $\ell_2$-norm in the interpolation regime (Theorems \ref{Th:convergConst} and \ref{Th:convergConstpop}), or in Wasserstein distance (towards the invariant distribution) in the case of a noisy system (Theorems \ref{thm:convergence_noisy} and \ref{thm:convergence_noisy_pop}). In the latter case, we further investigate the invariant distribution: we pinpoint its location (Proposition \ref{prop:localization}), and more importantly, we demonstrate in Proposition \ref{prop:momentexplo} that although there is no heavy tail phenomenon in finite time, it emerges asymptotically. We finally address convergence of \textit{variance reduction techniques} like time-averaging (Proposition~\ref{prop:ergodic}) and decay of step-size (Proposition \ref{prop:stepsizedecay}). Throughout the article, we try to present some theoretical background on the study of SDEs: among other the use of Lyapunov potentials and of coupling methods that enables to study quantitatively the speed of convergence to equilibrium. 

\paragraph{Further related work.} Formal links between the true stochastic gradient descent and its continuous model are studied in \cite{li2019stochastic} on the theoretical side, where a weak error analysis is provided, and in \cite{li2021validity} on the experimental side. In a similar article in spirit, \cite{fontaine2021convergence} provides convergence results on general convex function, but \textit{with a systematic polynomial step-size decay}. Some results implying the use of SDEs to study the influence of the noise on the \textit{implicit bias} the algorithm are given in the least-squares case in \cite{ali2020implicit}, and is an active topic in general~\cite{ziyin2021strength,wu2018sgd,pesme2021implicit,vivien2022label}. Quantitative studies of the invariant distribution focusing on the particular shape and intensity of the noise covariance include \cite{mori2022power,wojtowytsch2024stochastic}. Note that when the step-size is properly re-scaled with respect to the dimension, SGD converges in the high-dimensional limite to a SDE that is similar to the one we study here~\cite{ben2022high}, which is called \textit{homogenized SGD} in the least-square context~\cite{paquette2022homogenization}. Power laws of convergence toward stationarity related to the eigenvalue decay of the covariance matrix (capacity condition) are ubiquitous in statistics~\cite{caponnetto2007optimal} and the study of SGD for least-squares \cite{dieuleveut2016, berthier2020tight,pillaud2018statistical,cui2021generalization,bordelon2021learning}. Finally, the \textit{heavy-tail phenomenon}, has been re-discovered lately as an interesting feature of SGD with multiplicative noise~\cite{hodgkinson2021multiplicative,gurbuzbalaban2021heavy} and more recently in the context of SDE in the work of~\cite{jiao2024emergence}.

\paragraph{Organisation of the paper.} In Section~\ref{sec:set-up}, we present the general set-up of SGD both in the population and empirical cases as well as the possibility of the data \textit{in both cases} to be fully interpolated. Section~\ref{sec:SDE} explains the relevance of building a consistent SDE model of SGD and recall technical details related to SDEs. In Section~\ref{SGD_train}, the results concerning the dynamics of SGD in the training case are given, both in the interpolation regime (Section~\ref{subsec:SGD_noiseless}) and the noisy regime (Section ~\ref{subsec:SGD_noisy}). Section~\ref{SGD_online} is built similarly to the previous one and is devoted to online SGD. The proofs are postponed to the Appendix, for which precise references can be found in the main text.

\section{Set-up: Stochastic Gradient Descent on Least Squares Problems}
\label{sec:set-up}

In this section, we introduce the least-square problem that we consider throughout the article, putting emphasis on the difference between empirical distributions (training loss) and true ones (population loss). Nonetheless, the central argument of the article is that, aside from the known differences between empirical and test cases, the primary qualitative distinction hinges on whether the possibility of loss can be zero (interpolation), irrespective of the discrete nature of the distribution. The stochastic gradient descent is introduced at the end of the section.    

\subsection{The least-square problem: population and empirical losses.}

We consider a regression problem with input/output pair $(X,Y) \in \R^d \times \R$ distributed according to a joint probability law $\rho \in \mathcal{P}(\R^d \times \R)$. To learn the rule linking inputs to outputs, we take a linear family of predictors $\{ f_\theta : x \mapsto \langle \theta, x \rangle ,\, \theta \in \R^d\} $ and aim at minimizing the average \textit{square loss} on the penalty $\ell(\theta, (X, Y)): = \frac{1}{2}\big(\left\langle \theta, X \right\rangle - Y \big)^2 $
\begin{align}
\label{eq:population_loss}
    L(\theta) := \frac{1}{2} \E_{(X,Y)\sim \rho}\left[ \big(\left\langle \theta, X \right\rangle - Y \big)^2 \right]~.
\end{align}
\begin{remark}[Link with RKHS]
    Here, the family of predictor consists in linear functions of the input data $x \in \R^d$. While for the sake of clarity, we'll keep this linear family throughout the article, note that the same results apply for any family of linear predictors in an abstract reproducing kernel Hilbert space $(\mathcal{H}, \langle \cdot ,\cdot \rangle_{\mathcal{H}} )$, with feature $\R^d \ni x \mapsto \varphi(x) \in \mathcal{H} $, changing the dot product and family of predictors with the natural structure: $\{ f_\theta : x \mapsto \langle \theta, \varphi(x) \rangle_{\mathcal{H}},\, \theta \in \mathcal{H} \}$.
\end{remark}
In this article, in order to put emphasis on the possible difference between the two settings, we make a clear distinction between the \textit{population} and the \textit{empirical} cases. To keep notations simple, we refer to the population case whenever $\Omega:=\supp{(\rho)}$ is an open set of $\R^d$. On the contrary, we refer to the empirical case when $\rho$ is a finite sum of atomic measures, i.e. there exist $(x_1, y_1), \dots, (x_n, y_n) \in \R^d \times \R$ such that $\rho = \frac{1}{n} \sum_{i=1}^n \delta_{(x_i, y_i)}$. Obviously, the term empirical refers to the fact that, in this case, Eq.\eqref{eq:population_loss} can be seen as the training loss
\begin{align}
\label{eq:training_loss}
    L(\theta)=\frac{1}{2n} \sum_{i = 1}^n \left[ \big(\left\langle \theta, x_i \right\rangle - y_i \big)^2 \right]~.
\end{align}
In this case, the number $n \in \N^*$ will always denote the number of \textit{observed inputs/outputs pairs} $(x_i, y_i)_{i = 1 \dots n}$, and we will use the notations $\X = [x_1, \dots, x_n]^\top \in \R^{n \times d}$ to denote the design matrix, and $\y = (y_1, \dots, y_n)^\top \in \R^n$ to denote the output vector. With these notations, the training loss rewrites $L(\theta)=\frac{1}{2n} ||\X \theta - \y||_2^2$.

Note that even though for the sake of clarity we will sometimes treat them separately, these cases fit into the same framework and notations.

\subsection{Noisy and noiseless settings.}

For this, let us introduce the \textit{set of interpolators}:
\begin{align}
\label{eq:interpolators}
    \I = \left\{ \theta \in \R^d, \text{ such that } L(\theta) = 0 \right\}.
\end{align}
We distinguish between the two following settings.
\paragraph{Noisy setting $\boldsymbol{\I} \boldsymbol{=} \boldsymbol{\emptyset}$.} In this case, the loss $L$ is strictly lower bounded, i.e. there exists $\sigma > 0$ such that $L(\theta) \geq \sigma^2$ for all $\theta \in \R^d$. The two following examples show that this situation typically arises for different reasons in the population and the empirical cases. In the former, this is an effect of the noise model on the output, whereas in the latter, this is only a consequence of the underparametrization. 
\begin{example}[Noisy model]
\label{ex:pop_noisy}
Assume that there exists $\theta_* \in \R^d$ and $\xi \in \R$ a random variable independent of $X$ of mean zero and variance $2\sigma^2$, such that $Y = \langle \theta_*, X \rangle + \xi$. Then
\begin{align*}
    L(\theta)=\frac{1}{2} \left\| \Sigma^{1/2} (\theta - \theta_*) \right\|^2 + \frac{1}{2}\E(\xi^2) \geq \sigma^2~,
\end{align*}
where we use the notation for the input covariance matrix: $\Sigma:= \E_\rho [XX^\top] \in \R^{d \times d}$. 
\end{example}
\begin{example}[Underparametrized setting]
Consider the underparametrized regime for which $n > d$, and assume that each i.i.d. input/output couple $(x_i, y_i)_{i = 1 \hdots n}$ come from independent distributions that have densities (i.e. for all $ i \in \llbracket 1, n\rrbracket$, $x_i$ and $y_i$ are independent and are distributed according to laws that are absolutely continuous with respect to Lebesgue). Then, we have that almost surely $\I = \emptyset$.
\end{example}
\paragraph{Noiseless setting $\boldsymbol{\I} \boldsymbol{\neq} \boldsymbol{\emptyset}$.} This means that there exists at least one perfect linear interpolator of the model. In the population setting, this corresponds to the strong assumption that the model is well-specified and noiseless (formally $\xi = 0$, if we refer to Example~\ref{ex:pop_noisy}). Notably, this regime has received a large attention recently to model the large expressive power of neural network~\cite{ali2020implicit,berthier2020tight,vaswani2019fast}. Yet, in the empirical case, this typically and simply corresponds to the overparametrized regime~$d \geq n$.  
\begin{example}[Overparametrized setting]
Consider the overparametrized regime for which $d \geq n$, and assume that $((x_1, y_1), \dots, (x_n, y_n))$ are i.i.d. samples drawn from a distribution that has a density. Then, we have that, the zero-loss set is the affine set  $\I = \theta_* +\, \ker(\X)$, where $\theta_*$ is any element of $\I$. Furthermore, $\mathrm{dim}\, \I \geq d -n\geq 0 $ almost surely.
\end{example}

\subsection{The stochastic gradient descent}
The stochastic gradient descent (SGD) aims at minimizing a function through unbiased estimates of its gradient. While this method has been developed for a different purpose in the early 50's~\cite{robbins1951}, it is remarkable how SGD fits perfectly the modern large-scale machine learning framework~\cite{bottou2018optimization}. Indeed, the SGD iterative procedure corresponds to sample at each time $t \in \N^*$ an independent draw $(x_t, y_t)\sim \rho$ and update the predictor $\theta$ with respect to the local gradient calculated on this sample:
\begin{align}
    \theta_{t + 1} = \theta_t - \gamma \nabla_\theta \ell(\theta_t, (x_t, y_t)), 
\end{align}
where $\gamma > 0$ is the step size. 
Even if the population and empirical cases fall under the same general framework of stochastic approximation, let us detail the setups for these.
\paragraph{The population case: online SGD.} The population case corresponds to what is refer to as \textit{online SGD} in the literature~\cite{bottou2008tradeoffs}. In this case, for each time $t \in \N^*$, we have an independent draw $(x_t, y_t)\sim \rho$, and if we denote $\mathcal{F}_t = ((x_k, y_k),\, k \leq t)$ the natural adapted filtration, then $\E[\nabla_\theta \ell(\theta_t, (x_t, y_t))|\mathcal{F}_{t-1}] = \nabla L(\theta_t)$. That is $ \nabla_\theta \ell(\theta_t, (x_t, y_t))$ is an unbiased estimate of the true gradient of the risk and the recursion writes in explicit form
\begin{align}
\label{eq:SGD_test}
    \theta_{t + 1} = \theta_t - \gamma \left( \langle \theta_t, x_t \rangle - y_t \right) x_t.
\end{align}
Remark that even though at iteration~$t$, it has seen only $t$ samples, the SGD recursion directly optimizes the population loss~$L$. This is due to the fact that as long the recursion goes, we have access to \textit{fresh samples} from the distribution $\rho$. This can be seen informally as the case $n=+\infty$ and is in contrast to the finite~$n $ empirical case. A convenient way to present this dynamics is to force the apparition of the true gradient and rewrite the rest as a martingale increment. Indeed Eq.~\eqref{eq:SGD_test} reads
\begin{align}
\label{eq:SGD_test_martigale_form}
    \theta_{t + 1} &= \theta_t - \gamma \nabla_\theta L(\theta_t) + \gamma  \left(\nabla_\theta L(\theta_t) -\left( \langle \theta_t, x_t \rangle - y_t \right) x_t \right)\nonumber  \\
     &= \theta_t - \gamma \E_\rho\left[ \left( \langle \theta_t, X \rangle - Y \right) X \right] + \gamma m(\theta_t, (x_t, y_t))~,
\end{align}
where $m(\theta_t, (x_t, y_t)) := \E_\rho\left[ \left( \langle \theta_t, X \rangle - Y \right) X \right] - \left( \langle \theta_t, x_t \rangle - y_t \right) x_t$.

\paragraph{The empirical case: training SGD.} In this case, the main difference is that we consider a finite number of samples so that the SGD recursion will iterate over them uniformly at random. Mathematically, this corresponds to take at each time $t \in \N^*$ an independent sample of the uniform distribution~$i_t \sim \Unif(\{1, \dots,n\})$ and consider the unbiased estimate $\nabla_\theta \ell(\theta_t, (x_{i_t}, y_{i_t}))$ of the true gradient of the training loss Eq.\eqref{eq:training_loss}. For this, the adapted filtration writes $\mathcal{F}_t = ((x_k, y_k)_{k\leq n}, (i_k)_{k\leq t})$ and the recursion reads:
\begin{align}
\label{eq:SGD_train}
    \theta_{t + 1} = \theta_t - \gamma \left( \langle \theta_t, x_{i_t} \rangle - y_{i_t} \right) x_{i_t}.
\end{align}
Despite the fact that this dynamics looks very similar, the important difference is that, the batch of samples being fixed, the dynamics can select several times the same pair $(x_k, y_k)$. This is the reason why, informally after $t \geq \Theta(n)$ iterations, the training dynamics Eq.\eqref{eq:SGD_train} will deviate from the online SGD presented in Eq.\eqref{eq:SGD_test}. To end this section, let us also rewrite this dynamics with respect to the gradient descent plus the martingale increments:
\begin{align}
\label{eq:SGD_train_martigale_form}
    \theta_{t + 1} &= \theta_t - \gamma \nabla_\theta L(\theta_t) + \gamma \left( \nabla_\theta L(\theta_t) - \left( \langle \theta_t, x_{i_t} \rangle - y_{i_t} \right) x_{i_t} \right) \nonumber  \\
     &= \theta_t - \gamma \frac{1}{n}\sum_{i = 1}^n \left( \langle \theta_t, x_i \rangle - y_i \right) x_i  + \gamma m(\theta_t, (x_{i_t}, y_{i_t}))~,
\end{align}
where $m(\theta_t, (x_t, y_t)) := \frac{1}{n}\sum_{i = 1}^n \left( \langle \theta_t, x_i \rangle - y_i \right) x_i - \left( \langle \theta_t, x_{i_t} \rangle - y_{i_t} \right) x_{i_t}$. Note that these equations are in fact exactly the same as the ones presented in the population case, simply considering that the distribution is an empirical one, i.e. $\rho = \frac{1}{n}\sum_{i = 1}^n \delta_{(x_i, y_i)}$. We nonetheless decided to present them explicitly for the sake of clearness.

\section{Continuous model of SGD}
\label{sec:SDE}

In this section we give stochastic differential equations (SDEs) models for SGD. In a first time, we provide general necessary conditions to model well SGD. We instantiate them more precisely in a second time.

\subsection{The requirement of a SDE model}

As said above, the decomposition of the SGD recursion as in equations~\eqref{eq:SGD_test_martigale_form},\eqref{eq:SGD_train_martigale_form} is a generic feature of the stochastic descent~\cite{benaim2006dynamics}, as it has led to the celebrated \textit{ODE method}~\cite{harold1997stochastic} to study stochastic approximation of this type. Going further, this decomposition between \textit{a drift} $\nabla L$ and \textit{local martingale term} $m(\cdot, (x,y))$ is reminiscent of the decomposition occuring for Itô processes, i.e. solution to a SDE of the type 
\begin{align}
    \label{eq:general_SDE}
    d \theta_t = b(t, \theta_t) dt + \sqrt{\gamma}\sigma(t, \theta_t) d B_t,
\end{align}
where $(B_t)_{t\geq 0}$ is a Brownian motion of $\R^d$. These types of models have been largely studied in the literature~\cite{khasminskii2011stochastic}, as, beyond their large modelling abilities, they offer useful tools, e.g. Itô calculus, when it comes to mathematical analysis. One of the aim of this article is to link the SGD dynamics to processes that are exemplary in the SDE literature. In order to establish this link, we first have to answer the following question
\vspace{-0.05cm}
\begin{center}
    \textit{What SDE model fits well with the SGD dynamics? \\ $\ $ \vspace{-0.1cm}}
\end{center}
This question has received a large attention in the last decade, and a good principle to answer this is to turn to \textit{stochastic modified equations}~\cite{li2019stochastic}. This is a natural way to build models of SDE since they are consistent in the infinitesimal step-size limit with SGD. In order to build such model, there are two requirements:
\begin{enumerate}[label=(\roman*)]
    \item The drift term $b(t, \theta_t)$ should match $-\nabla L(\theta_t)$.
    \item The noise factor $\sigma$ should have the same covariance as the local martingale $m$, i.e.
    \begin{align}
        \sigma(t, \theta_t) \sigma(t, \theta_t)^\top = \E_\rho \left[ m(\theta_t, (x_t, y_t)) m(\theta_t, (x_t, y_t))^\top\, \big|\, \mathcal{F}_{t-1} \right].
    \end{align}
\end{enumerate}
Besides technical assumptions, these are the two requirements presented in~\cite[Theorem 3]{li2019stochastic} to show that the SDE model is \textit{consistent} in the small step-size limit with the SGD recursion. Going beyond the approximation concerns tackled by (i) and (ii), it has recently been observed that the SGD noise carries a specific shape that the SDE model should carry as well~\cite{wojtowytsch2023stochastic,pesme2021implicit}. This requirement has a more \textit{qualitative nature} but is important to fully capture the essence of the SGD dynamics 
\begin{enumerate}[label=(\roman*)]
    \item[(iii)] The noise term $\sigma(t, \theta_t) dB_t$ should span the same space as $m(\theta_t, (x_t, y_t))$.  
\end{enumerate}

We will see below, that, in order to build a SDE model of SGD, this third requirement is particularly important in the empirical case, where the noise a strong degeneracy. 

\subsection{Explicit form of the SDE models}

Let us first write explicitly the multiplicative noise factor $\sigma(t, \theta_t)$ in the population and empirical cases. Then, we derive the expression of the SDE models that we analyze later. 
\paragraph{The population case.}
In the population case, the calculation has already been made in~\cite{ali2020implicit} and, defining $r_{X}(\theta) = \langle \theta_t, X \rangle - Y \in \R $, the residual random variable, then we have that 
    \begin{align*}
        \sigma(t, \theta_t) \sigma(t, \theta_t)^\top =  \left( \E_\rho \left[ r_{X}(\theta_t)^2 X X^\top \right] - \E_\rho \left[ r_{X}(\theta_t) X\right] \E_\rho \left[ r_{X}(\theta_t) X\right]^\top \right).
    \end{align*}
In this case, there is no geometric specificity of the noise and we choose to $\sigma(t, \theta_t)$ as the PSD square root of the right hand side, that is, we define, for all $\theta \in \R^d$, 
    \begin{align}
    \label{eq:def_sigma_population}
        \sigma(\theta) := \left(\E_\rho \left[ r_{X}(\theta)^2 X X^\top \right] - \E_\rho \left[ r_{X}(\theta) X\right] \E_\rho \left[ r_{X}(\theta) X\right]^\top \right)^{1/2} \in \R^{d \times d},
    \end{align}
and we study the following SDE model:
\begin{align}
    \label{eq:population_SDE}
    d \theta_t = - \E_\rho \left[ \left( \langle \theta_t, X \rangle - Y \right) X \right] dt + \sqrt{\gamma} \sigma(\theta_t) d B_t,
\end{align}
where $(B_t)_{t\geq 0}$ is a Brownian motion of $\R^d$ and $\sigma \in \R^{d\times d}$ is given by \eqref{eq:def_sigma_population}.

\paragraph{The empirical case.}

The empirical case is, from the geometric perspective, a bit more subtle. Indeed, as can be seen directly in Eq.\eqref{eq:SGD_train}, the iterates of SGD stay in the low-dimenional space of dimension at most $n$: $\theta_0\, +\, \vspan(x_1, \dots, x_n) \subset \R^d$. Hence, it is primordial for a good SDE model that the noise carries this degenerate structure. Let us apply for the SDE noise factor calculation. Let us define $r_{\mathsf{x}}(\theta) := ( \langle \theta_t, x_1 \rangle - y_1, \dots, \langle \theta_t, x_n \rangle - y_n)^\top \in \R^n $, the vector of the residuals, then  
    \begin{align*}
        \sigma(t, \theta_t) \sigma(t, \theta_t)^\top & =  \left( \frac{1}{n} \X^\top\,\diag(r_{\mathsf{x}}(\theta_t))^2\X - \frac{1}{n^2}\X^\top\, r_{\mathsf{x}}(\theta_t)\, r_{\mathsf{x}}(\theta_t)^\top\X \right) \\
        & = \frac{1}{n}\X^\top\, \left(\diag(r_{\mathsf{x}}(\theta_t))^2 - \frac{1}{n}r_{\mathsf{x}}(\theta_t)\, r_{\mathsf{x}}(\theta_t)^\top \right) \X \\
        & = \frac{1}{n}\X^\top\, \left(\diag(r_{\mathsf{x}}(\theta_t)) - \frac{1}{n} r_{\mathsf{x}}(\theta_t)\mathds{1}^\top \right) \left(\diag(r_{\mathsf{x}}(\theta_t)) - \frac{1}{n} r_{\mathsf{x}}(\theta_t)\mathds{1}^\top \right)^\top \X~,
    \end{align*}
where for any $\theta \in \R^d$, we define $\diag(\theta) \in \R^{d \times d}$ as the diagonal matrix with the coordinates of $\theta$ as diagonal entries and $\mathds{1} \in \R^n$ the vector of ones.  
In the case where $n < d$, the overall covariance matrix is degenerate as its rank is at most~$n$: this is because the ``noise of SGD'' belong naturally to $\mathrm{span}(x_1, \hdots, x_n)$. Hence, it is important to keep this degeneracy and this is the reason why we choose \textit{the} square root that preserves it:
    \begin{align*}
        \sigma(\theta_t) := \frac{1}{\sqrt{n}}  \X^\top R_\x (\theta_t) \in \R^{d \times n},
    \end{align*}
where we defined $R_\x (\theta) := \diag(r_{\mathsf{x}}(\theta_t)) - \frac{1}{n} r_{\mathsf{x}}(\theta_t)\mathds{1}^\top \in \R^{n \times n}$. With these notations, we study the following SDE model:
\begin{align}
    \label{eq:train_SDE}
    d \theta_t = - \frac{1}{n}\X^\top(\X \theta_t - \y) dt + \sqrt{\frac{\gamma}{n}}  \X^\top R_\x (\theta_t) d B_t,
\end{align}
where $(B_t)_{t\geq 0}$ is a Brownian motion of $\R^n$.
\paragraph{Initial condition and moments.} Initialization will be taken at some $\theta_0 \in \R^d$ which can be considered as a random variable. Standard choices for the law of $\theta_0$ include the standard Gaussian of $\R^d$, or a dirac measure on some vector $\theta_0$, e.g. $\theta_0 = 0$. Either cases, its law $\rho_0$ has moments of all orders, and since drift and multiplicative noise of the SDEs \eqref{eq:population_SDE}-\eqref{eq:train_SDE} have at most linear growth, Theorem 3.5 of \cite{khasminskii2011stochastic} shows that the marginal laws $(\rho_t)_{t \geq 0}$ all have moments at all orders at any time $t \geq 0$. 

Now that we motivated the SDE models in the previous part, and stated them in Eqs.\eqref{eq:population_SDE}-\eqref{eq:train_SDE}, we study their convergence property. This is the purpose of the main results that are stated in the two following sections. 

\section{SGD on the training loss}
\label{SGD_train}
Recall that this corresponds to the finite data set case with $n$ input/output pairs $(x_i, y_i)_{i = 1\hdots n}$, stacked into data matrix $\X \in \R^{n \times d}$ and data output vector $\y \in \R^d$. We study in this section the SDE~given in equation~\eqref{eq:train_SDE}.
%
%
We assume that all data are bounded, i.e. there is some $\K > 0$ such that $\|x_i\|^2 \leq \K$  for all $i \in \llbracket 1, n\rrbracket$. Let us introduce first an important element of the model. In this section, we define
\begin{definition}
    Let $\X^\dagger$ denote the \textit{pseudo-inverse} of the design matrix $\X$. We define 
    \begin{align}
    \label{eq:least_square_estimator}
        \theta_* = \X^\dagger \y~ + (I - \X^\dagger \X) \theta_0~. 
    \end{align}
\end{definition}
Note that, in generic cases, in the underparametrized case for which $d \leq n$, we have  $\X^\dagger =  (\X^\top \X)^{-1} \X^\top$. Then, $\X^\dagger \X = I$, and $\theta_* = \X^\dagger y$ is the \textit{Ordinary least-square estimator} and does not depend on $\theta_0$. In the generic overparameterized case for which $n \leq d$, we have $\X^\dagger = \X^\top (\X \X^\top)^{-1} $, and hence $\X \theta_* = \X\X^\dagger \y + (\X - \X \X^\dagger \X) \theta_0 = \y $, that is to say that $\theta_* \in \mathcal{I}$. 

Finally, for both cases, we define $\Sigma := \frac{1}{n}\X^\top \X \in \R^{d \times d}$ the design covariance matrix. With this notation, the SDE becomes
\begin{align*}
    d \theta_t = - \Sigma(\theta_t - \theta_*) dt + \sqrt{\frac{\gamma}{n}}  \X^\top R_\x (\theta_t) d B_t,
\end{align*}
where $(B_t)_{t\geq 0}$ is a Brownian motion of $\R^n$.
\paragraph{Comparison with standard processes.} As already said, the dynamics in the noiseless and noisy cases are of different natures because of the possibility to cancel the multiplicative noise term. Indeed, for clarity imagine that $n = d$ and $\X / \sqrt{n} = \Sigma = I_d$. If the noise can cancel, $R_\x (\theta_t)$ has the shape of a linear term like $\diag(\theta - \theta_*)$, and each the coordinate of the difference $\eta = \theta - \theta_* $ follows a one-dimensional \textit{Geometric Brownian motion}
$
    d\eta_t = - \eta_t dt + \sqrt{\gamma} \eta_t d B_t~,
$
for which it is known that $\eta_t \to 0$ almost surely, which corresponds to $\theta_t \to \theta_*$.
This comparison is the governing principle of the analysis of this setup.
Otherwise, if the noise cannot cancel and is strictly lower bounded, then, under the same proxy, the movement of each coordinate of $\eta$ resembles the SDE  
$
    d\eta_t = - \eta_t dt + \sqrt{\gamma} \sqrt{\eta_t^2 + \sigma^2} d B_t
$
which looks like a Ornstein-Uhlenbeck process if $\eta \ll \sigma$, but with a noise that has a multiplicative part $\eta_t d B_t$ when $\eta \gg \sigma$. These are known in one dimension under the name of \textit{Pearsons diffusions}~\cite{forman2008pearson} and exhibit stationary distribution with heavy tails.

The difference between these two settings is the reason why we divide the results into two different subsections.

\subsection{The noiseless case}
\label{subsec:SGD_noiseless}

In this section, we let $(\theta_t)_{t \geq 0}$ follow the dynamics given by Eq.\eqref{eq:train_SDE} initialized at $\theta_0 \in \R^d$ and assume that $\ker(\X) \neq \{0\}$ and  that $\mathcal{I} \neq \emptyset$. Note that this generically occurs in the overparametrized case for which $n \leq d$. Recall that we have defined $\theta_* = \X^\dagger \y~ + (I - \X^\dagger \X) \theta_0$ in Eq.\eqref{eq:least_square_estimator}, we present now a geometric interpretation of $\theta_*$ in this case,
\begin{lemma}
\label{lem:projection}
    The vector $\theta_*$ is the orthogonal projection of $\theta_0$ into $\I$, that is 
    \begin{align}
        \theta_* = \underset{\X \theta = y}{\mathrm{argmin}} \, \, \|\theta - \theta_0\|^2~.
    \end{align}
\end{lemma}
The proof is deferred to Section~\ref{appsec:noiseless_train} of the Appendix.
Note that this fact is often known as the implicit bias of least-squares methods and has statistical consequences studied under the name of ``benign overfitting''~\cite{bartlett2020benign,bartlett2021deep}.

Remarkably, despite the randomness of the SDE~\eqref{eq:train_SDE}, the convergence is almost sure towards~$\theta_*$, with explicit rates that we show below.

\begin{theorem} \label{Th:convergConst}
    Let $(\theta_t)_{t \geq 0}$ follows the dynamics given by Eq.\eqref{eq:train_SDE} initialized at $\theta_0 \in \R^d$, then for $\gamma <  \frac{1}{3K}$, we have that $(\theta_t)_{t \geq 0} $ converges almost surely to $\theta_*$ with the following rates
    \begin{enumerate}[label=\bfseries(\roman*)]
        \item {\bfseries Parametric rate.} For all $ t \geq 0$, we have that
        \begin{align}
            \E[\|\theta_t - \theta_*\|^2] \leq \|\theta_0 - \theta_*\|^2 e^{-\mu (2-\K\gamma)  t},
        \end{align}
         where $\mu>0$ is the smallest non zero eigenvalue of $\Sigma$. 
        \item {\bfseries Non-parametric rate.} For all $ t \geq 0$, we have that for all $\alpha > 0$, 
        \begin{align}
            \E[\|\theta_t - \theta_*\|^2] \leq \left[ \frac{1}{\| \theta_0 - \theta_*\|^{-2/\alpha} + \Consa t}\right]^{\alpha}~,
        \end{align} 
        where
        $ \Consa =\frac{1}{2\alpha}( \langle \theta_0  - \theta_*,  \Sigma^{-\alpha} (\theta_0  - \theta_*)   \rangle + \frac{\gamma \K_{\alpha}}{2-\K \gamma } \|\theta_0 - \theta_*\|^2)^{-1/\alpha}$, and  $K_{\alpha}=\max_{i\leq n}\langle x_i, \Sigma^{-\alpha} x_i \rangle.$
    \end{enumerate}
\end{theorem}
Hence, for all $ t \geq 0$, 
$$\displaystyle \E[\|\theta_t - \theta_*\|^2] \leq \min \left\{\|\theta_0 - \theta_*\|^2 e^{-\mu t},  \inf_{\alpha \in \R^+}\left[ \frac{1}{\| \theta_0 - \theta_*\|^{-2/\alpha} + \Consa t}\right]^{\alpha}\right\}~,$$ 
and, if $\mu$ is non-zero but very small (e.g. $10^{-10}$), this inequality describes well the difference between the transcient regime of convergence that is polynomial and the asymptotic regime that is exponential but occuring after time-scale $1/\mu$. These polynomial convergence bounds relying on the covariance upperbound, $K_\alpha$, are common in the study of SGD in RKHS~\cite{dieuleveut2016,lin2016optimal,berthier2020tight}. Second, by Markov inequality, remark that the estimates (i) and (ii) give convergence rates in probability. Furthermore, one steadily checks that one can derive the same kind of inequality as in (i) for the expected iterates $(\E[\theta_t])_{t \geq 0}$. In this case, one obtains the overestimation $ \|\theta_0 - \theta_*\|^2 e^{-2\mu   t}$ and from the perspective of the rate given by the Theorem, the stochastic nature of the dynamic seems to slow down the convergence. Finally, note that considering $\theta_0$ as a random variable, all expectations can be thought of as conditional expectations $\E[\,\cdot\, | \theta_0]$. The results are illustrated in Figure~\ref{fig:convergence_noiseless} and \ref{fig:convergence_noisy}.

\begin{figure}
\centering
  \includegraphics[width=0.6\linewidth]{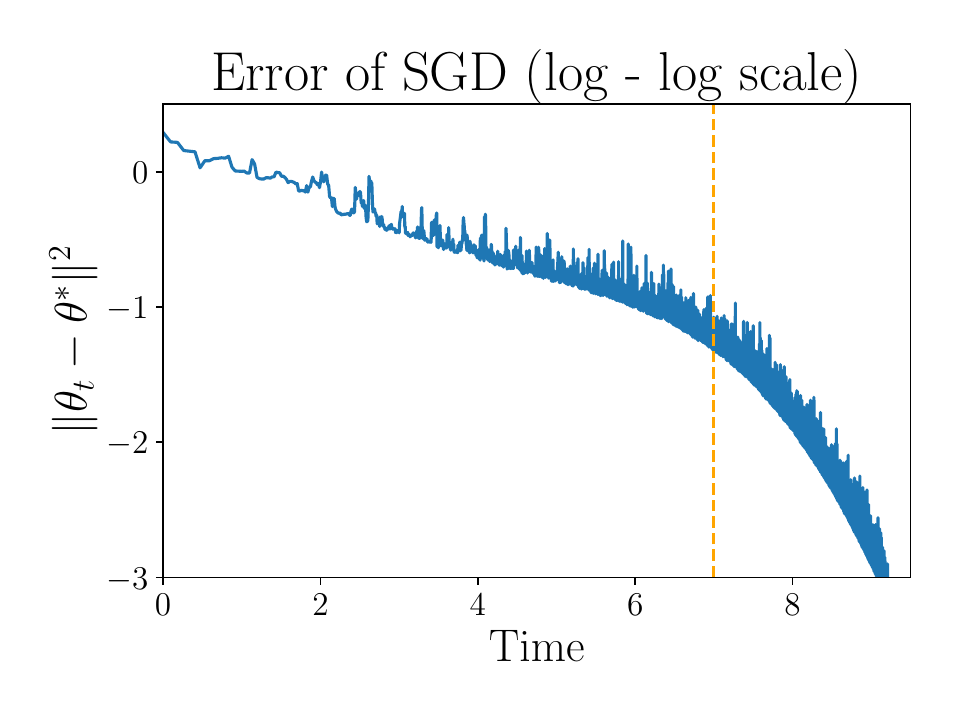}
  \caption{Plot showing the error of SGD along time for an overparametrized regime where $n = 100$ and $d = 200$. The samples $(x_i)_{i\leq n}$ come from a Gaussian distribution with a covariance whose eigenvalues decay as a power law. The vertical dotted orange line illustrates the separation between the two regimes depicted by Theorem~\ref{Th:convergConst}, the polynomial one before (a straight line in a log-log plot) and the exponential line, after typical time scale~$1/\mu$. This illustrates perfectly the rates of convergence shown in Theorem \ref{Th:convergConst}.}
  \label{fig:convergence_noiseless}
\end{figure}

\begin{figure}
  \centering
  \includegraphics[width=0.6\linewidth]{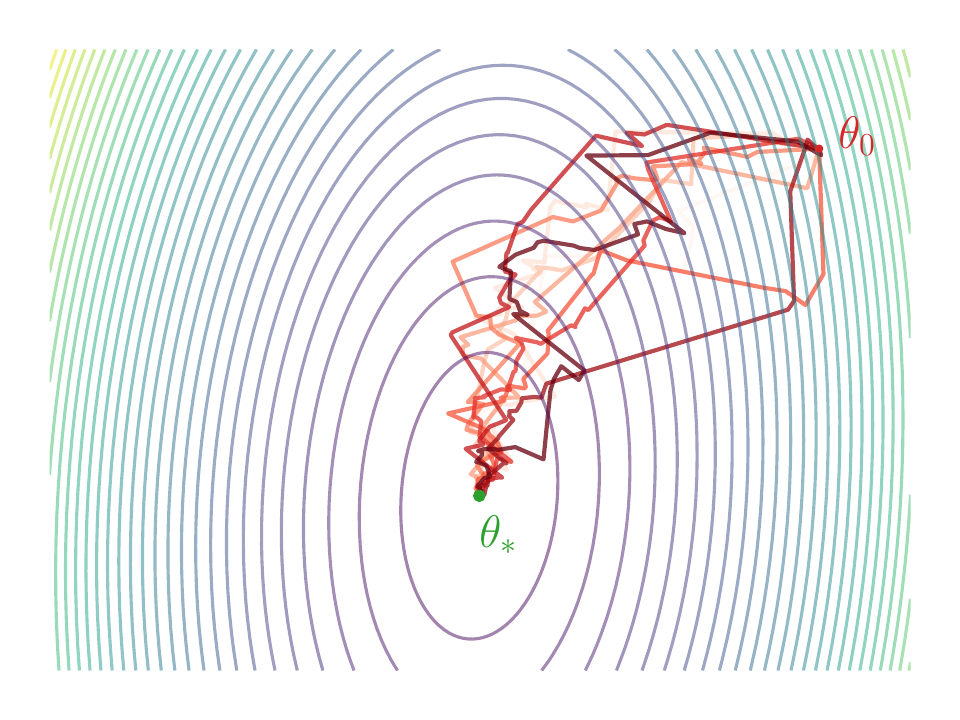}
  \caption{Display of a two-dimensional projection of $10$ trajectories of SGD for $n = 100$, $d = 200$, in the case that there is an perfect interpolator $\theta_*$. The ellipses represent the level curves of the training loss.}
  \label{fig:convergence_noisy}
\end{figure}

\subsection{The noisy case}
\label{subsec:SGD_noisy}

In this section too, $(\theta_t)_{t \geq 0}$ follows the dynamics given by Eq.\eqref{eq:train_SDE} initialized at $\theta_0 \in \R^d$. However, we assume in this subsection that $\ker \X = \{0\}$ and that $\mathcal{I} = \emptyset$. 
Note that this is generically the case if $n \geq d$. The important consequence of this  is that the loss is uniformly lower bounded by some positive number, indeed, we have
\begin{lemma}
    Let $\sigma^2 = L(\theta_*)$, we have, for all $\theta \in \R^d$ that $ L (\theta) \geq \sigma^2$.
\end{lemma}
\begin{proof}
    The function $\theta \mapsto L(\theta)$ is convex, hence any minimizer of $L$ satisfies the equation $\nabla L(\theta) = 0$. This is the case of $\theta_*$, as $\nabla L (\theta) = \Sigma (\theta - \theta_*)$.
\end{proof}
Hence the noise covariance matrix is (most of the time) uniformly bounded away from zero, and this implies that there exists a unique invariant measure to the SDE equation. This in contrast to the noiseless case ($\I \neq \emptyset$) where any measure that is supported in $\I$ is naturally invariant.

\paragraph{Semi-group.} Due to the non-degeneracy of the noise term in this case, it is more convenient to track the dynamics of the probability measure $\rho_t:= \text{Law}(\theta_t)$ for any $t \geq 0$, which is the time marginal of the process initialized at $\theta_0$ and defined such that $\E [f(\theta_t)] = \int f d\rho_t $, for any $f: \R^d \to \R $ smooth enough. In fact the SDE Eq.\eqref{eq:train_SDE} has an associated semi-group $\P_t$, defined so that for all $\eta \in \R^d$, $(\P_t f)(\eta) := \E[f(\theta_t)\, |\,\theta_{t = 0} = \eta]$. Formally, if $\rho_t$ has a density, we have $\rho_t(\theta) = (\P_t \delta_\theta)(\theta_0) $ for all $\theta \in \R^d$. Note that by local strong ellipticity, even if $\rho_0 = \delta_{\theta_0}$ is singular, regularization properties of SDEs ensure that at $t = 0^+$, the measure $\rho_{t}$ has a density with respect to Lebesgue as well as a second order moment~\cite[Section 4 of Chapter 9]{friedman2008partial}. We place ourselves in this setup where $\rho_0 \in \mathbb{P}_2(\R^d)$, the set of probability measures $\mu$ such that $\int \|\theta\|^2 d\mu(\theta) < + \infty $.

\paragraph{Infinitesimal generator.} It is known that the time-marginal law of $(\theta_t)_{t \geq 0}$ satisfies the (parabolic) Fokker-Planck equation (at least in the weak sense):
\begin{align*}
    \partial_t \rho_t = \L^* \rho_t~,
\end{align*}
with the operator $\L^*$ being defined as, for all $\theta \in \R^d$
\begin{align*}
   (\L^* f)(\theta)  = \text{div}\left[\Sigma\left(\theta - \theta_*\right) f(\theta) \right] + \frac{\gamma}{2n}\sum_{i,j = 1}^d \partial_{ij} \left( \left[\X^\top R^2_{\x}(\theta) \X\right]_{ij}  f(\theta)\right),
\end{align*}
whose adjoint (with respect to the canonical dot product of $L^2(\R^d)$) is often referred to as the \textit{infinitesimal generator} of the dynamics and writes, 
\begin{align}
    (\L f) (\theta) = -  \langle \Sigma\left(\theta - \theta_*\right), \nabla f(\theta) \rangle  + \frac{\gamma}{2n}\sum_{i,j = 1}^d [\X^\top R^2_{\x}(\theta) \X]_{ij} \partial_{ij} f(\theta),
\end{align}
for all test functions $f: \R^d \to \R$ sufficiently smooth. Recall furthermore that the evolution of expectations of \textit{observables} $(f(\theta_t))_{t \geq 0}$ is given by the Dynkin forrmula~\cite[Lemma 3.2]{khasminskii2011stochastic}: $ \frac{d}{dt}\E \left[f(\theta_t)\right] = \L \E \left[f(\theta_t)\right]$.
This identity enables, as in the deterministic case for gradient flow, the use of Lyapunov functions. This is a very useful tool to study the asymptotic behavior of stochastic processes. This is the objective of the following lemma:
\begin{lemma}
    \label{lem:Lyapunov_noisy}
    Let $V (\theta) = \frac{1}{2} \|\theta - \theta_*\|^2$, we have the inequality, for all $\theta \in \R^d$, 
    \begin{align}
        \mathcal{L} V (\theta) \leq - 2(1-\gamma \K/2) L (\theta) + 2\sigma^2~.    
    \end{align}
     In consequence, for $\gamma \leq 1/(3\K)$, there exists a stationary process to the SDE \eqref{eq:train_SDE}.
\end{lemma}
The proof is postponed to
the Appendix \ref{appsec:noisy_training}.
\subsubsection{Invariant measure and convergence.} 

Any invariant measure $\rho_\infty$ satisfies the (elliptic) Partial Differential Equation (PDE): $\L^* \rho_\infty = 0$. As said before if the multiplicative noise does not cancel, $\L^*$ is uniformly elliptic and the uniqueness of such a measure is ensured. This is the case generically if $n \geq 2d$, however, for the sake of completeness we prove that in any case there exists a unique invariant measure (at this expense of a bit more technicality). Moreover, by ergodicity, the law of the iterates eventually converges towards this unique invariant measure. In the following result we also provide a quantitative statement, in Wassertein distance, on the speed of convergence of the dynamics towards such a measure.
\begin{theorem} \label{thm:convergence_noisy}
    Let $(\theta_t)_{t \geq 0}$ follows the dynamics given by Eq.\eqref{eq:train_SDE} initialized at $\theta_0$ distributed according to $\rho_0 \in \mathbb{P}_2(\R^d)$, then for $\gamma <  \frac{1}{3\K}$ , there exists a unique stationary distribution $\rho^* \in \mathbb{P}_2(\R^d)$, and quantitatively,
    \begin{enumerate}[label=\bfseries(\roman*)]
        \item {\bfseries Parametric rate.} For all $ t \geq 0$, we have that
        \begin{align}
            \W_2^2(\rho_t, \rho^*) \leq \W_2^2(\rho_0, \rho^*) e^{-2\mu( 1 - 2\gamma \K )  t},
        \end{align}
         where $\mu>0$ is the smallest non zero eigenvalue of $\Sigma$.
        \item {\bfseries Non-parametric rate.} For all $ t \geq 0$, we have that $ \forall \alpha > 0$,
        \begin{align}
            \W_2^2(\rho_t, \rho^*)  \leq \left[ \frac{1}{\W_2^2(\rho_0, \rho^*)^{-2/\alpha} + \Consa t}\right]^{\alpha}~,
        \end{align} 
        with 
        $\Consa =\frac{2 \left(1 - 2\gamma \K\right)}{\alpha}\left(\E\left[ \langle \theta_0 - \Theta_*, \Sigma^{-\alpha} (\theta_0 - \Theta_*) \rangle \right]+ \frac{2 \gamma\K_{\alpha}}{1-2\gamma\K}\W_2^2(\rho_0, \rho^*))\right)^{-1 / \alpha}$, and where the expectation is taken w.r.t. the optimal Wasserstein coupling between $ \theta_0 \sim \rho_0$ and $\Theta_* \sim \rho_*$.
    \end{enumerate}
\end{theorem}
The proof of these results is deferred in Appendix, Section~\ref{appsubsec:Invariant measure} and rests on coupling techniques standard in the asymptotic behavior of SDEs~\cite{von2005transport} for the Wasserstein distance. One could adapt the argument to derive rates of convergence for other \textit{probabilistic distances} that have a coupling representations like the total variation~\cite{hairer2010convergence}; however, due to the multiplicative noise, convergence rates in the ``natural metric" given by $L^2(\rho^*)$ are more difficult to obtain. This could be the occasion for a future investigation (see the recent work of \cite{cattiaux2023journey} for a possibility to overcome this). 
Similarly as for the noiseless cases, there is a difference between time-scale of convergence between the early regime $t \ll 1/\mu$ that is polynomial and the asymptotic time $t \geq 1/\mu$ that is exponentially fast. 
\begin{figure}[ht] 
  \begin{minipage}[b]{0.5\linewidth}
    \centering
    \includegraphics[width = \linewidth]{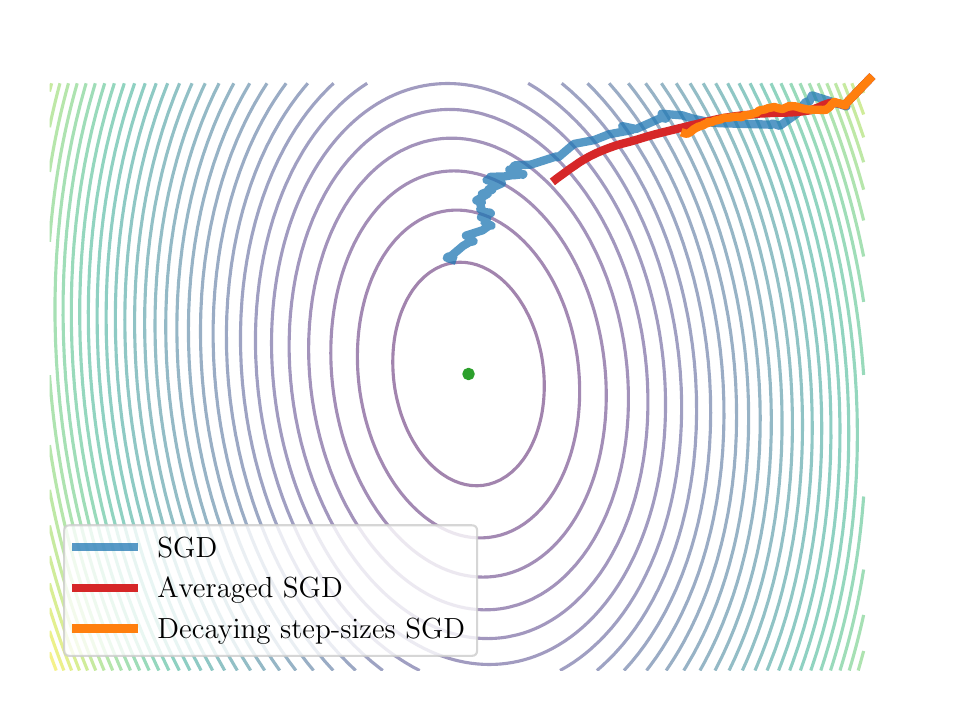} 
    \vspace{-3em}
  \end{minipage}
  \begin{minipage}[b]{0.5\linewidth}
    \centering
    \includegraphics[width = \linewidth]{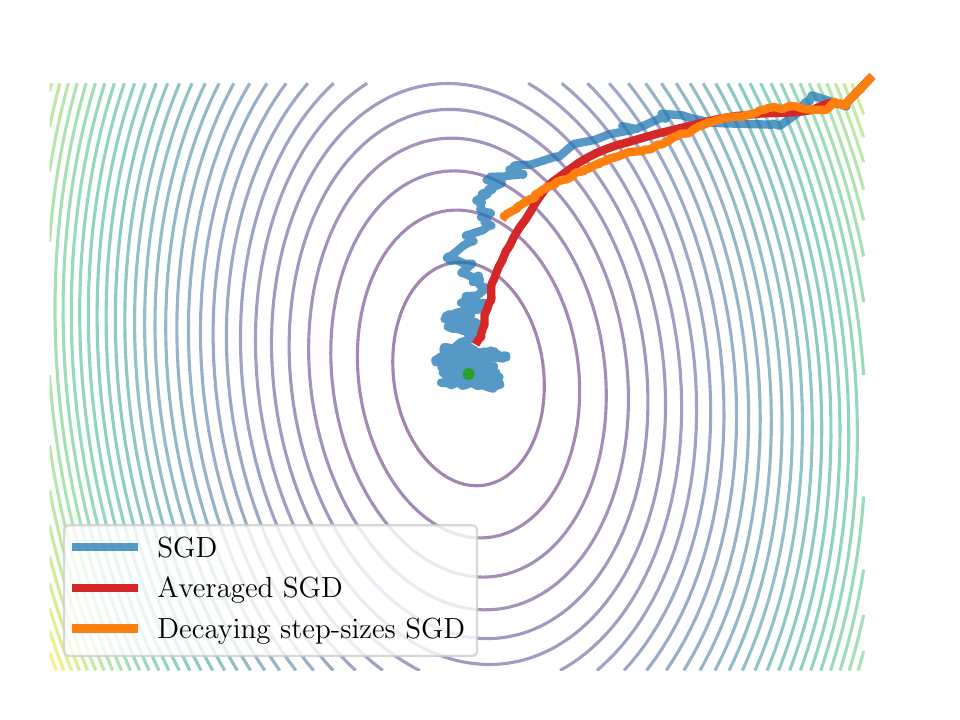} 
    \vspace{-3em}
  \end{minipage} 
  \begin{minipage}[b]{0.5\linewidth}
    \centering
    \includegraphics[width = \linewidth]{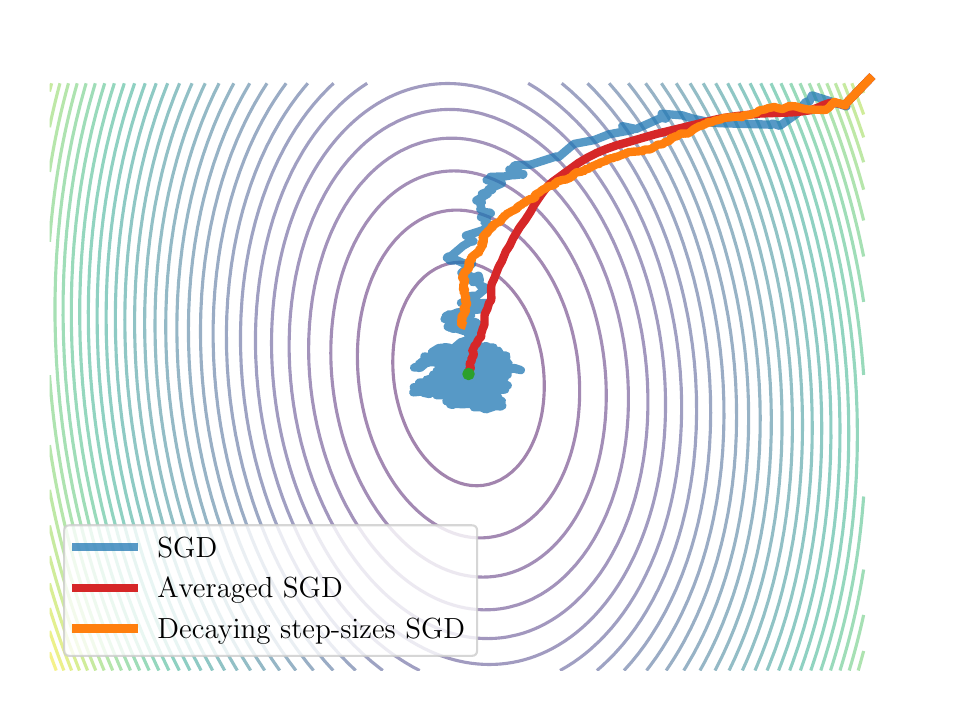} 
    \vspace{-3em}
  \end{minipage}
  \begin{minipage}[b]{0.5 \linewidth}
    \centering
    \includegraphics[width = \linewidth]{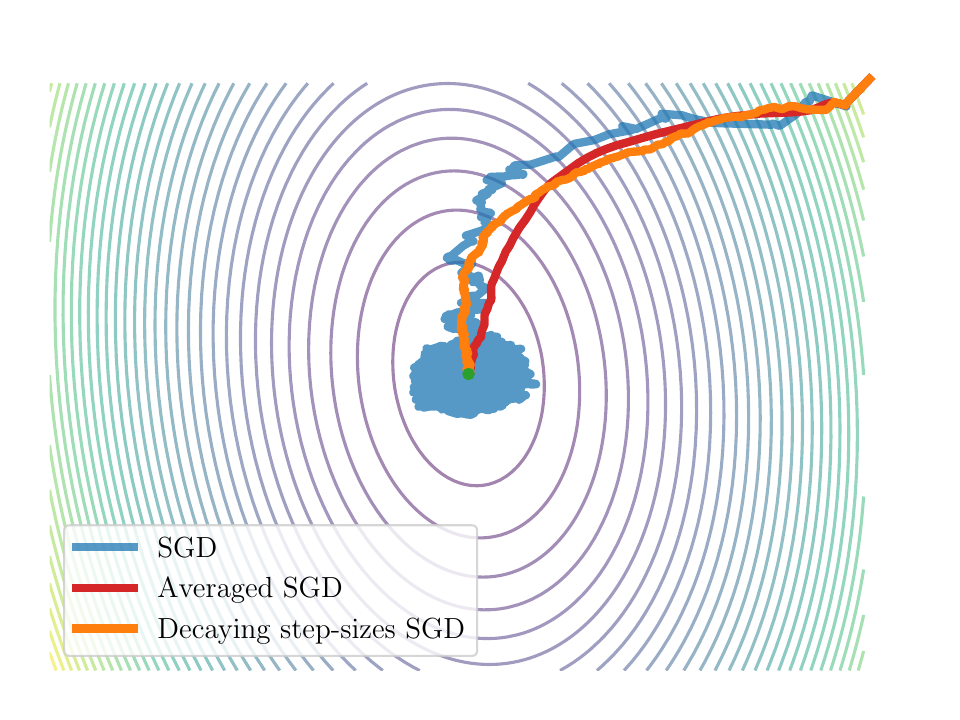} 
    \vspace{-3em}
  \end{minipage} 
  \caption{Four plots showing the trajectory of SGD in the noisy setting. The arrow of time goes from top left to bottom right. We see that the two variance reduction methods (time average and decaying step-sizes) converge towards $\theta_*$ (confirming Propositions~\ref{prop:ergodic} and \ref{prop:stepsizedecay}), while plain SGD has a stationary distribution with certain fluctuations around its mean $\theta_*$ as explained in Theorem~\ref{thm:convergence_noisy} and Proposition~\ref{prop:localization}. Plain SGD is faster to its invariant distribution as the variance reduction methods as shown in the convergence rates provided in the results.}
  \label{fig:sgd_noisy}
\end{figure}

\subsubsection{Localization of the invariant measure.} 

Above, we have shown that the law of the iterates converges towards a measure $\rho^*$ at a certain speed. We know that the invariant measure satisfies the equation $\L^* \rho^* = 0$; yet, this does not really give any practical description or insight on it. Notably, we want to understand how its localization depends on $\theta_*$ and on the parameters $\sigma^2, \gamma \text{ or } \X$. This is the purpose of the following proposition.
%


\begin{proposition}
\label{prop:localization}
    Let $\Theta_*\sim\rho^*$, then for $\gamma < \frac{1}{3\K}$ we have 
    \begin{align}
      \E \left[  \Theta_*  \right] = \theta_*~,  \text{ and } \ \E \left[ \left\| \Theta_* - \theta_* \right\|^2 \right] \leq \frac{ \gamma \K  \sigma^2}{\mu(1-\gamma \K )}~.
    \end{align}
\end{proposition}
The proof can be found in the Appendix, Section \ref{appsubsec:Localization}. Note that, as expected, the deviation of $\Theta_*$ to its mean goes to $0$ as $\sigma$ or $\gamma$ goes to~$0$. Moreover, in the limit $\gamma \ll 1$, we have  $\E \left\| \Theta_* - \theta_* \right\|^2\sim \gamma \K \sigma^2/\mu$, which reflects, up to constants, the scale given by the Gaussian calculation provided in Remark~\ref{rmk:gaussian_case}.

\begin{remark}
\label{rmk:gaussian_case}
    In the case for which we consider that the residuals equilibrate near the noise level $\sigma>0$, we can model that $R_\x \simeq \sigma I_n$, so that the SDE dynamics rewrites 
    \begin{align*}
    d \theta_t = - \frac{1}{n}\X^\top(\X \theta_t - \y) dt + \sqrt{\frac{\gamma \sigma^2}{n}}  \X^\top d B_t.
    \end{align*}
    This simple model is particularly convenient as it allows to compute in close form the invariant distribution. Indeed, solving $\L^* \rho^* = 0$ in this case gives for all $\theta \in \R^d$, 
    \begin{align*}
        \rho^*(\theta) \propto \mathrm{e}^{-\frac{\|\theta-\theta_*\|^2}{\gamma\sigma^2}}.
    \end{align*} 
    That is to say that $\rho^*$ is the Gaussian of mean $\theta_*$ and covariance $\displaystyle \frac{\gamma\sigma^2}{2} I_d$.
    \end{remark}
However, recall that $R_\x(\cdot)$ is in fact linear, this implies that for large $\|\theta\|$,  the noise in the SDE amplifies: this has consequences on the tails of the distribution, see next paragraph. 
\paragraph{Heavy-tails or not?} Recall that at any time $t>0$ all the moments of the law of $\theta_t$ exist (provided the initial distribution had such moments). Yet, depending on the step-size, moments of the invariant distribution might not exist. More precisely, we show that if one fixes a step-size $\gamma > 0$, all moments up to a certain value $\alpha(\gamma)$ of the invariant distribution exist and all higher moments do not. In other words the step size is a direct control of the tail of the asymptotic distribution.

\begin{proposition}
\label{prop:momentexplo} For $n\geq 2d$ and a fixed $\gamma<\frac{1}{3\K}$, it exists $\alpha>0$ large enough such that 
\[
\E(\|\Theta_*\|^{\alpha})=+\infty.
\]
\end{proposition}
The proof of this result can be found in the Appendix, Section~\ref{appsubsec:Localization}. This result is in accordance of the fact that multiplicative noise in SGD induces heavy-tails of its asymptotic distribution~\cite{hodgkinson2021multiplicative,gurbuzbalaban2021heavy,jiao2024emergence}.  
\subsubsection{Convergence of variance reduction techniques}
If the aim is to have an estimator that converges at long times to the optimal $\theta_*$, it is necessary to use \textit{variance reduction techniques}: here we show how time averages and step-size decay can be used in order to achieve this.
\paragraph{Ergodicity: convergence of the time average.} Classical ergodic theorems tell that the time average along a trajectory converges towards spatial mean taken according to the invariant measure, i.e. $\frac{1}{t} \int_0^t \varphi(\theta_s) ds \to \E_{\rho^*} \left[ \varphi(\Theta_*) \right]$, as $t$ goes to infinity, for every smooth function $\varphi$. Taking $\varphi = I$, we thus expect that $\bar{\theta}_t := \frac{1}{t} \int_0^t \theta_s ds \to \E_{\rho^*} \left[ \Theta_* \right] = \theta_*$, as $t$ goes to infinity, as explained in Proposition~\ref{prop:localization}. We quantify the convergence speed in the following.  
\begin{proposition}\label{prop:ergodic}
    We have for all $t \geq 0$ that 
    \begin{align*}
        \E \| \Sigma (\bar{\theta}_t - \theta_*) \|^2 &\leq \frac{8 \gamma \K  \sigma^2 }{t}  + \frac{10 \| \theta_0 - \theta_* \|^2}{t^2}.
    \end{align*}
\end{proposition}
%

%
%
\paragraph{Step-size decay.} 
In this section, the step-size $\gamma=\gamma_t$ depends on the time $t$ and tends to $0$ as $t$ tends to $\infty$. In this case, $(\theta_t)_{t \geq 0}$ eventually converges.
\begin{proposition}
\label{prop:stepsizedecay}
    Let $\gamma_t=1/(2\K+t^{\alpha})$ and $\alpha>1$, we have
    \begin{align*}
        \E [\| \theta_t - \theta_* \|^2] &\leq \frac{C_\alpha}{t^{\alpha - 1}},
    \end{align*}
    with $C_\alpha = e^{-\alpha}\left(\E [\| \theta_0 - \theta_* \|^2]e(2(\alpha-1)/\mu)^{\alpha-1} +(2\alpha/\mu)^{\alpha}   \sigma^2\right) + \frac{2^{1+\alpha}\K \sigma^2}{(\alpha-1)}$.  
\end{proposition}
This result of polynomial decay when the step-size decay polynomially is similar to the one provided in the work of~\cite{fontaine2021convergence}, and is re-proven here for the sake of completeness. 
For both results, the proofs are postponed to the Appendix, Section~\ref{appsubsec:Reductiontechniques}. All results on convergence are illustrated in Figure~\ref{fig:sgd_noisy}.

\section{Online SGD (population loss)}
\label{SGD_online}

In the previous section, on SGD in the \textit{empirical setting}, we have shown that the core of the qualitative description relies on the fact that the noise can cancel or not. The situation is similar for online SGD. In fact most of the calculation that we have shown so far transfer almost immediately within this setting. This is the reason why we will only show the results concerning the convergence in the \textit{noisy} and the \textit{noiseless} settings and describe briefly what properties will be similar than the ones exhibited in the previous section.

First, we recall the main equation governing the dynamics \eqref{eq:population_SDE}:
\begin{align*}
    d \theta_t = - \E_\rho \left[ \left( \langle \theta_t, X \rangle - Y \right) X \right] dt + \sqrt{\gamma} \sigma(\theta_t) d B_t,
\end{align*}
where $(B_t)_{t\geq 0}$ is a Brownian motion of $\R^d$ and $\sigma \in \R^{d\times d}$ is given by
\begin{align*}
        \sigma(\theta) := \left(\E_\rho \left[ r_{X}(\theta)^2 X X^\top \right] - \E_\rho \left[ r_{X}(\theta) X\right] \E_\rho \left[ r_{X}(\theta) X\right]^\top \right)^{1/2} \in \R^{d \times d}.
\end{align*}
Let us define $\Sigma = \E_{\rho}[ X X^\top] \in \R^{d \times d}$ the covariance matrix that we assume invertible and $\theta_* =  \Sigma^{-1} \E_{\rho}[ Y X] \in \R^d$. We have that for all $\theta \in \R^d$, 
$\E_\rho \left[ \left( \langle \theta, X \rangle - Y \right) X \right] = \Sigma (\theta - \theta_*)$, and hence the dynamics writes:
\begin{align*}
    d \theta_t = - \Sigma (\theta_t - \theta_*) dt + \sqrt{\gamma} \sigma(\theta_t) d B_t.
\end{align*}
Despite the difference between the empirical and population learning setups, the story here is similar to what is describe in Section~\ref{SGD_train}. Indeed, the core of the behavior of the algorithm relies only on the fact that $\theta_*$ cancels also the noise matrix $\sigma$, i.e $\sigma(\theta_*) = 0$. If this is the case, i.e. we are in the interpolation (or noiseless) regime, then the same analysis as in Subsection~\ref{subsec:SGD_noiseless} applies and the movement resembles multivariate geometric Brownian motion. If uniformly $\sigma(\theta) \succ 0$, then there will be a non degenerate invariant measure and the same results as Subsection~\ref{subsec:SGD_noisy} go through.

\subsection{The noiseless case}
Here we assume that the distribution $\rho$ is such that $Y = \langle X, \theta_* \rangle$ and that the features are almost surely bounded, i.e., $\|X\|^2 \leq \K$ for some $\K > 0$. This means that we are in the noiseless/interpolation regime where the input/output distribution admits a linear interpolator. As previously, in this case, despite the randomness of the SDE, the convergence is almost sure towards~$\theta_*$, with explicit rates that we show below.

\begin{theorem} \label{Th:convergConstpop}
    Let $(\theta_t)_{t \geq 0}$ follows the dynamics given by Eq.\eqref{eq:train_SDE} initialized at $\theta_0 \in \R^d$, then for $\gamma <  \frac{1}{3\K}$, we have that $\|\theta_t - \theta_*\|$ converges almost surely to $0$ with the following rates
    \begin{enumerate}[label=\bfseries(\roman*)]
        \item {\bfseries Parametric rate.} For all $ t \geq 0$, we have that
        \begin{align}
            \E[\|\theta_t - \theta_*\|^2] \leq \|\theta_0 - \theta_*\|^2 e^{-\mu (2-\K\gamma)  t},
        \end{align}
         where $\mu>0$ is the smallest non zero eigenvalue of $\Sigma$. 
        \item {\bfseries Non-parametric rate.} For all $ t \geq 0$, we have that for all $\alpha > 0$, 
        \begin{align}
            \E[\|\theta_t - \theta_*\|^2] \leq  \left[ \frac{1}{\| \theta_0 - \theta_*\|^{-2/\alpha} + \Consa t}\right]^{\alpha}~,
        \end{align} 
        where we defined 
        $ \Consa =\frac{1}{2\alpha}( \langle \eta_0,  \Sigma^{-\alpha} \eta_0   \rangle + \frac{\gamma \K_{\alpha}}{2-\K \gamma } \|\theta_0 - \theta_*\|^2)^{-1/\alpha}$, and $K_{\alpha}=\max_{i\leq n}\langle x_i, \Sigma^{-\alpha} x_i \rangle$.
    \end{enumerate}
\end{theorem}
Hence, as in the case of the training dynamics, we have that for all $ t \geq 0$, $$\displaystyle \E[\|\theta_t - \theta_*\|^2] \leq \min \left\{\|\theta_0 - \theta_*\|^2 e^{-\mu t},  \inf_{\alpha \in \R^+}\left[ \frac{1}{\| \theta_0 - \theta_*\|^{-2/\alpha} + \Consa t}\right]^{\alpha}\right\}~,$$ and, if $\mu$ is non-zero but very small (e.g. $10^{-10}$), this inequality describes well the difference between the transcient regime of convergence that is polynomial and the asymptotic regime that is exponential but occuring after time-scale $1/\mu$. We recalled here the result as in the previous section for the sake of completeness, but the result and the proof of the theorem (which can be found in Appendix, Section~\ref{appsec:noiseless_pop}) are rather similar to the training case. This shows that what really matters is the ability of the model to be interpolated or not and not the finite sample size.

\subsection{The noisy case}
\label{subsec:SGD_noisy_population}

In this section we suppose that we are in the noisy regime. This corresponds to the fact that there exists a constant $a > 0$ such that for all $\theta \in \R^d$, we have $\sigma^2(\theta)/L(\theta) \succcurlyeq a^2 I_d$ as well as $L(\theta) \geq a^2$, that is to say that neither the loss nor the multiplicative noise can cancel.  Let us give first a concrete example when this happens.

\begin{example}[Gaussian model]
\label{ex:pop_noisy_gaussian}
Assume that $X \sim \mathcal{N}(0, \Sigma)$, with $\Sigma \succeq \mu I_d$ and that there exists $\theta_* \in \R^d$ and $\xi \in \R$ a random variable independent of $X$ of mean zero and variance $2\sigma^2$, such that $Y = \langle \theta_*, X \rangle + \xi$. Then, by some forth order Gaussian moment calculation given in \cite[Lemma~1]{berthier2020tight}, we have the exact calculation
\begin{align*}
        \sigma(\theta)^2 = \left(\Sigma (\theta- \theta_*)\right)\left(\Sigma (\theta- \theta_*)\right)^\top + 2 L(\theta) \Sigma \succcurlyeq 2 \mu L(\theta) I_d~,
\end{align*} 
as well as $L(\theta) \geq \sigma^2$.
\end{example}
\paragraph{Semi-group and Fokker-Planck equation.} As for the training setup, due to the non-degeneracy of the noise term in this case, it is more convenient to track the dynamics of the probability measure $\rho_t:= \text{Law}(\theta_t)$ for any $t \geq 0$. We place ourselves in the situation $\rho_0 \in \mathbb{P}_2(\R^d)$ and recall that $(\theta_t)_{t \geq 0}$ satisfies the (parabolic) Fokker-Planck equation (at least in the weak sense) $\partial_t \rho_t = \L^* \rho_t$,
with the operator $\L^*$ being the adjoint (with respect to the canonical dot product of $L^2(\R^d)$) of the infinitesimal generator of the dynamics that writes, 
\begin{align}
    (\L f) (\theta) = -  \langle \Sigma\left(\theta - \theta_*\right), \nabla f(\theta) \rangle  + \frac{\gamma}{2}\sum_{i,j = 1}^d [\sigma(\theta) \sigma(\theta)^\top]_{ij} \partial_{ij} f(\theta),
\end{align}
for all test functions $f: \R^d \to \R$ sufficiently smooth. Recall furthermore that the evolution of expectations of \textit{observables} $(f(\theta_t))_{t \geq 0}$ is given by the action of the generator:
\begin{align}
    \frac{d}{dt}\E \left[f(\theta_t)\right] = \L \E \left[f(\theta_t)\right]~.
\end{align}
\subsubsection{Invariant measure and convergence.} As the multiplicative noise does not cancel, $\L^*$ is uniformly elliptic and there is existence and uniqueness of the invariant measure $\rho_\infty$ of the SDE, which satisfies the PDE $\L^* \rho_\infty = 0$. Moreover, by ergodicity, the law of the iterates eventually converges towards this unique invariant measure. In order to show this quantitavely, we first prove a useful lemma on the Lipschitz behavior of the multiplicative noise:
\begin{lemma}
\label{lem:noise_lip}
    There exists of constant $c > 0$ depending only on the distribution $\rho$, such that for all $\theta, \eta \in \R^d$, we have
    \begin{align}
        \| \sigma(\theta) - \sigma(\eta)  \|^2_{\mathrm{HS}} &\leq 2 c \K \langle \Sigma(\theta - \eta), \theta - \eta \rangle~.
    \end{align}
\end{lemma}
The proof is in the Appendix, Section~\ref{appsec:noise_lip}.
We are now ready to state the main theorem of the section. Indeed, in the following result we also provide a quantitative statement, in Wassertein distance, on the speed of convergence of the dynamics towards such a measure.
\begin{theorem}\label{thm:convergence_noisy_pop}
    Let $(\theta_t)_{t \geq 0}$ follows the dynamics given by Eq.\eqref{eq:train_SDE} initialized at $\theta_0 \in \R^d$, then for $\gamma <  \frac{1}{\K c}$ , there exists a unique stationary distribution $\rho^* \in \mathbb{P}_2(\R^d)$, and quantitatively,
    \begin{enumerate}[label=\bfseries(\roman*)]
        \item {\bfseries Parametric rate.} For all $ t \geq 0$, we have that
        \begin{align}
            \W_2^2(\rho_t, \rho^*) \leq \W_2^2(\rho_0, \rho^*) e^{-2\mu( 1 - \gamma \K c)  t},
        \end{align}
         where $\mu>0$ is the smallest non zero eigenvalue of $\Sigma$.
        \item {\bfseries Non-parametric rate.} Assume that we have the inequality for all $\alpha > 0$,
        \begin{align}
        \label{eq:lip_chelou}
             \| \Sigma^{-\alpha/2} (\sigma(\theta) - \sigma(\eta))  \|^2_{\mathrm{HS}} &\leq 2 c_\alpha \K_{\alpha} \langle \Sigma(\theta - \eta), \theta - \eta \rangle~,
        \end{align}
        then, for all $ t \geq 0$, we have that $ \forall \alpha > 0$,
        \begin{align}
            \W_2^2(\rho_t, \rho^*)  \leq \left[ \frac{1}{\W_2^2(\rho_0, \rho^*)^{-2/\alpha} + \Consa t}\right]^{\alpha}~,
        \end{align} 
        with 
        $\Consa =\frac{2 \left(1 - \gamma c \K\right)}{\alpha}\left(\E\left[ \langle \theta_0 - \Theta_*, \Sigma^{-\alpha} (\theta_0 - \Theta_*) \rangle \right]+ \frac{ \gamma c
        _\alpha \K_{\alpha}}{1-\gamma c\K}\W_2^2(\rho_0, \rho^*))\right)^{-1 / \alpha}$, and where the expectation is taken w.r.t. the optimal Wasserstein coupling between $ \theta_0 \sim \rho_0$ and $\Theta_* \sim \rho_*$.
    \end{enumerate}
\end{theorem}

Similarly as before, the result is similar to the one of the underparametrized regime for the training dynamics. This shows once again that what really matters is the ability of the model to be interpolated or not and not the finite sample size. The proof of the theorem can be found in Appendix, Section \ref{appsec:noise_lip}.
\subsubsection{Further (expected) properties.} In the following paragraph we state what should be the properties of the dynamics similarly to the SDE on the training loss case. Any proof could be adapted, but we take the option to avoid being redundant and state directly the expected properties without re-doing the work of the previous section. To be fair we do not state the results as propositions and theorems. We will also abuse of the notation $\mathsf{c}> 0$ that plays the role of a universal constant in any of the expressions below.
\vspace{0.15cm}

\noindent \textit{It is possible to localize the invariant measure} similarly to what has been done in Proposition~\ref{prop:localization}. We give below estimates on its mean and standard deviation around it. Indeed,  if $\Theta_*$ be a random variable distributed according to $\rho^*$, then we expect that 
    \begin{align}
      \E \left[  \Theta_*  \right] = \theta_*~,  \text{ and } \ \E \left[ \left\| \Theta_* - \theta_* \right\|^2 \right] \leq  \mathsf{c}\frac{ \gamma \K  \sigma^2}{\mu}~.
    \end{align}
\vspace{-0.25cm}

\noindent \textit{We expect from ergodicity that time averages} along a trajectory converge towards spatial mean taken according to the invariant measure, i.e. $\bar{\theta}_t := \frac{1}{t} \int_0^t \theta_s ds \to \E_{\rho^*} \left[ \Theta_* \right] = \theta_*$ , as $t$ goes to infinity, as stated in the equation above. A quantification of this fact would give:
\begin{align*}
    \E \| \Sigma (\bar{\theta}_t - \theta_*) \|^2 &\leq \mathsf{c} \frac{ \gamma \K  \sigma^2 +  \| \theta_0 - \theta_* \|^2 }{t}.
\end{align*}
\vspace{-0.25cm}

\noindent \textit{Step-size decay} would help canceling the noise in large times: indeed, choosing the step-size sequence as $\gamma_t= 1/(\K+t^{\alpha})$ for $\alpha>1$, we get
    \begin{align*}
        \E [\| \theta_t - \theta_* \|^2] &\leq  \frac{\mathsf{c}}{t^{\alpha - 1}},
    \end{align*}
\vspace{-0.25cm}

\noindent \textit{The heavy-tails phenomenon} is  also expected to hold similarly to the training case. More precisely, we show that if one fixes a step-size $\gamma > 0$, all moments up to a certain value $\alpha(\gamma)$ of the invariant distribution exist and all higher moments do not. In other words the step size is a direct control of the tail of the asymptotic distribution.

\section{Conclusion and Perspectives}

In this article, we have shown how the SGD could be efficiently modeled by a SDE to reflect its main qualitative \textit{and} quantitative features: convergence speed, difference between the noisy and noiseless settings and study of the asymptotic distribution. The specificity of the least-square set-up enabled us to show some \textit{localization} of the invariant measure $\rho_*$ in the noisy context: however, it seems possible to improve this understanding to its precise \textit{shape}, and some questions remain: is $\rho_*$ log-concave? Is it possible to characterize its covariance in order to better apprehend its shape? Also, regarding its heavy-tail behavior, it would be great to have a precise estimate of the exponent, from which the moment of the stationary distribution explodes. 

This work has been done in order to convey the idea that the SDE framework can improve the understanding of the SGD dynamics. This is clear for the least-squares setup, yet the important question is to go beyond this and try to apply the same methodology for the non-convex dynamics arising from the training of non-linear neural networks. The study of \textit{single or multi-index models}~\cite{ben2022high,bietti2023learning} could be a first step toward broadening this systematic study.

\vspace{0.25cm}

{\noindent \bfseries Acknowledgments.} AS acknowledges support by the Deutsche Forschungsgemeinschaft (DFG, German Research Foundation) under Germany’s Excellence Strategy – GZ 2047/1, Projekt-ID 390685813. AS and LP warmly thank the the Simons Foundation and especially the Flatiron Institute for its support as this research has been initiated while AS was invited in New York. Finally, the authors extend their gratitude to the \textit{Incubateur de Fraîcheur} for hosting them and providing an ideal atmosphere that fostered exceptional discussions.

\bibliographystyle{unsrt}
\bibliography{SDE_SGD.bib}

\clearpage

\appendix

\begin{center}
{\LARGE \textsc{Organization of the Appendix}}    
\end{center}

\vspace{1cm}

\noindent \textbf{\ref{AppendixA}}. \textbf{SGD on the training loss: Proof of Section~\ref{SGD_train}}.\\[0.5 cm]
 \textbf{\ref{appsec:noiseless_train}}: \textbf{The noiseless case}. 
 This includes the proof of the fact that $\theta_*$ is the orthogonal projection of $\theta_0$ into $\I$ (Lemma \ref{lem:projection}) and the proof of the convergence theorem of $\theta$ to $\theta_*$ (Theorem \ref{Th:convergConst}).\\[0.2 cm]
 \noindent \textbf{\ref{appsec:noisy_training}}: \textbf{The noisy case}. We prove here the existence of a stationary distribution (Lemma \ref{lem:Lyapunov_noisy}).\\[0.2 cm]
  \noindent \textbf{\ref{appsubsec:Invariant measure}}: \textbf{Invariant measure and convergence}. In this subsection, we prove the quantitative convergence to the the stationary distribution (Theorem \ref{thm:convergence_noisy}) with the help of a technical Lemma (Lemma \ref{lem:almost_def_pos}). \\[0.2 cm]
  \noindent \textbf{\ref{appsubsec:Localization}}: \textbf{Localization of the invariant measure}. We prove the insights given on the first and second moments of the stationary distribution (Proposition \ref{prop:localization}). Then, we show the moment explosion of the invariant distribution (Proposition \ref{prop:momentexplo}).  \\[0.2 cm]
  \noindent \textbf{\ref{appsubsec:Reductiontechniques}}: \textbf{Convergence of variance reduction techniques}. In this subsection, we prove the convergence of the time-average of the iterates to $\theta_*$, i.e. ergodicity (Proposition \ref{prop:ergodic}) as well as the convergence of $\theta_t$ to $\theta_*$ in the case of the step-size decay (Proposition \ref{prop:stepsizedecay}). \\[0.5 cm]
 \noindent \textbf{\ref{AppendixB}}: \textbf{Online SGD: Proofs of Section~\ref{SGD_online}.}\\[0.5 cm]
 \noindent \textbf{\ref{appsec:noiseless_pop}}: \textbf{The noiseless case} We give the proof of the convergence to $\theta_*$ with rates (Theorem \ref{Th:convergConstpop}).\\[0.2 cm]
 \noindent \textbf{\ref{appsec:noise_lip}}: \textbf{The noisy case}. We first prove that multiplicative noise carries some Lipschitz property (Lemma \ref{lem:noise_lip}). This allow us to prove the quantitative convergence to the the stationary distribution (Theorem \ref{thm:convergence_noisy_pop}).\\[0.2 cm]

\vspace{0.25cm}

\section{SGD on the training loss: Proofs of Section~\ref{SGD_train}} \label{AppendixA}

\subsection{The noiseless case}
\label{appsec:noiseless_train}

We begin by proving Lemma~\ref{lem:projection} on the fact that 
$\theta_*$ is the orthogonal projection of $\theta_0$ into $\I$, that is $\theta_* = \underset{\X \theta = y}{\mathrm{argmin}} \, \, \|\theta - \theta_0\|^2~.$
\begin{proof}[\textbf{Lemma~\ref{lem:projection}}]
    The proof of the lemma follows from Karush–Kuhn–Tucker conditions. Indeed, the argmin is unique, being the projection of $\theta_0$ on a linear set, and it satisfies that there exists Lagrange multipliers $\lambda \in \R^n$ such that
\begin{align*}
    \theta_* - \theta_0 = \X^\top \lambda~, \ \text{and} \ X\theta_* = y~.
\end{align*}
This means that $ \lambda = (\X \X^\top)^{-1} (y - \X \theta_0) $, and hence, 
\begin{align*}
    \theta_* &= \theta_0+  \X^\top (\X \X^\top)^{-1} (y - \X \theta_0)  \\
    &= \X^\dagger y +  (I - \X^\dagger \X) \theta_0~,
\end{align*}
which corresponds to the given definition of $\theta_*$.
\end{proof}

We prove now the main convergence theorem of this section. This is Theorem~\ref{Th:convergConst}.
\\
\begin{proof}[\textbf{Theorem~\ref{Th:convergConst}}]
\noindent $\bf (i)$ The key ingredient of the proof is the Gronwall Lemma. Combining the Itô formula with \eqref{eq:train_SDE} gives us
    \begin{align*}
       \frac{d}{dt} \E \| \theta_t - \theta_*\|^2 &=  -2 \E \langle \theta_t - \theta_*,  \Sigma (\theta_t - \theta^*)  \rangle + \frac{\gamma}{n} \E \Tr\left( \X \X^\top R^2_\x (\theta_t)\right) \\
         &\leq  -2 \E \langle \theta_t - \theta_*,  \Sigma (\theta_t - \theta^*)  \rangle + \frac{\gamma}{n} \E \sum_{i=1}^n \|x_i\|^2 (\langle \theta_t , x_i \rangle - y_i)^2 \\
         &\leq  -2 \E \langle \theta_t - \theta_*,  \Sigma (\theta_t - \theta^*)  \rangle + \gamma \K \E \langle \theta_t - \theta_*,  \Sigma (\theta_t - \theta^*)  \rangle \\
         &\leq  -(2 - \gamma \K) \E \langle \theta_t - \theta_*,  \Sigma (\theta_t - \theta^*)  \rangle~.
    \end{align*}
       By integrating the latter, we get
    \begin{align*}
       \E \| \theta_t - \theta_*\|^2 &=  \| \theta_0 - \theta_*\|^2  - \frac{2}{n} \int_{0}^t \E \| \X \left(\theta_u - \theta_* \right) \|^2   du   \\
        &\hspace{2.5cm} + \frac{\gamma}{n} \int_{0}^{t} \E\Tr\left(\X^\top \left(\diag(r_{\mathsf{x}}(\theta))^2 - \frac{1}{n}r_{\mathsf{x}}(\theta_t)\, r_{\mathsf{x}}(\theta_t)^\top \right) \X \right) du .
    \end{align*}
    Note that 
    \begin{align*}
    \frac{\gamma}{n} \int_{0}^{t} \E\Tr\left(\X^\top \left(\diag(r_{\mathsf{x}}(\theta))^2  \right) \X \right) du &= \frac{\gamma}{n} \int_{0}^{t} \E \sum_{i=1}^n r_{\mathsf{x}}^2(\theta_t)_i ( \X  \X^\top)_{ii} du \\
    &= \frac{\gamma}{n} \int_{0}^{t} \E \sum_{i=1}^n r_{\mathsf{x}}^2(\theta_t)_i \|x_i\|^2 du  \\
    &\leq \frac{\gamma}{n} \int_{0}^{t} \E \max_{i}{\|x_i\|^2} \sum_{i=1}^n r_{\mathsf{x}}^2(\theta_t)_i  du  \\
    &\leq 2 \gamma \K \int_{0}^{t} \E L(\theta_u)  du, 
    \end{align*}
          and 
          \begin{multline*}
              \frac{\gamma}{n^2} \int_{0}^{t} \E\Tr\left(\X^\top r_{\mathsf{x}}(\theta_t)\, r_{\mathsf{x}}(\theta_t)^\top \X \right) du = \frac{\gamma}{n^2} \int_{0}^{t} \E\Tr\left( r_{\mathsf{x}}(\theta_u)^\top \X\X^\top r_{\mathsf{x}}(\theta_u)\,  \right)  du \\
              \geq \frac{\gamma}{n} \lambda_{min}({\Sigma}) \int_{0}^{t} \E  \left(r_{\mathsf{x}}(\theta_t)^\top  r_{\mathsf{x}}(\theta_t)\,  \right) du= 2 \gamma   \lambda_{min}({\Sigma}) \int_{0}^{t} \E  L(\theta_u) du.
               \end{multline*}
Collecting the estimates, we obtain 
 \begin{equation*}
       \E \| \theta_t - \theta_*\|^2 \leq  \| \theta_0 - \theta_*\|^2  - \left(4+2 \gamma  \lambda_{min}({\Sigma})-2 \gamma \K\right)  \int_{0}^t \E L(\theta_u)   du.
    \end{equation*}

We remark that $\theta_u - \theta_* \in \text{Ran}(X^T)$. Indeed, we have
 \[
 \theta_u - \theta_*=(\theta_u - \theta_0)+(\theta_0 - \theta_*),
 \]
the first term on the r.h.s. is in $\text{Ran}(X^T)$ by \eqref{eq:train_SDE} and the second term also by \eqref{eq:least_square_estimator}. We recall that ${\R^d}=\text{Ran}(X^T)\oplus \text{Ker}(X)$ and therefore note that $X_{|\text{Ran}(X^T)}$ is a bijection into its image. Combining the two last facts yield
 \[
\frac{\lambda_{min}(\Sigma)}{2} \E  \|  \left(\theta_u - \theta_* \right) \|^2 \leq \frac{1}{2n}\E  \| \X \left(\theta_u - \theta_* \right) \|^2 =   \E L(\theta_u) 
\]
because $\X\X^\top$ as the same spectrum as $\X^\top\X$. All in all, we have
 \begin{equation*}
        \E \| \theta_t - \theta_*\|^2 \leq  \| \theta_0 - \theta_*\|^2 - \lambda_{min}(\Sigma) \left(2+ \gamma   \lambda_{min}({\Sigma})- \gamma K\right) \int_{0}^t \E \| \left(\theta_u - \theta_* \right) \|^2   du  , 
    \end{equation*}
  and the first statement of the Theorem follows with Gronwall Lemma. We move to the proof of the second statement of the Theorem.
    \\
\noindent $\bf (ii)$ We define $\eta_t:=\theta_t - \theta_*$ and recall that 
    \begin{equation*}
        \E \| \eta_t\|^2 \leq  \| \theta_0 - \theta_*\|^2  - 2(2-\gamma \K)\int_{0}^t \E L(\theta_u)   du.
    \end{equation*}
    The first consequence of this inequality is that
    \begin{align*}
        \int_{0}^t \E L(\theta_u) du \leq  \frac{1}{2(2-\gamma \K)}\| \theta_0 - \theta_*\|^2~.
    \end{align*}
    
    We want to lower bound the term $\E L(\theta_u)$ without using the smallest eigenvalue of $\Sigma$, that is supposedly infinitely small. First note $\X^\top \X$ is symmetric and thus diagonalizable in an orthonormal basis $(v_i)$. We steadily check that
    \[
    \text{Ker}(\X^\top \X)=\text{Ker}( \X)
    \]
    by invertibility of $\X\X^\top$. We thus have  ${\R^d}=\text{Ran}(X^T)\oplus \text{Ker}(\X^\top \X)$. We denote by $\lambda_1 \leq \dots \leq \lambda_{d^*}$ the non zero eigenvalue of $\X^\top \X$ where $d^*\leq d$ is the number of non-zero eigenvalues and $(v_1, \hdots, v_{d^{*}})$ the corresponding eigenvectors. For all $t  > 0$, we thus  have the following decomposition  $\eta_t = \sum_{k = 1}^{d^*} \eta^t_k v_k$. Thanks to the Hölder inequality, we claim that the following holds true:\\
        Let $p,q \in (0,1), \, p+q=1$, we have, for all $t \geq 0$, 
        \begin{equation}\label{eq:hoelder}
            \left(\E \left[\|\eta_t\|^2\right]\right)^{\frac{1}{p}} \leq 2\E \left[ L(\theta_t)\right]  \left(\E \left[\langle  \eta_t,  \Sigma^{-p/q} \eta_t   \rangle \right]\right)^{q/p}~.
        \end{equation}
 We move to the proof of the above claim. To lighten the notation, we write $\eta_k^t=\eta_k$ in the following to wit
    \begin{align*}
        \E \left[\|\eta_t\|^2\right]  &=  \sum_{k=1}^{d^*} \E\left[\eta_k^2\right] =  \sum_{k=1}^{d^*} \E\left[\eta_k^{2p} \lambda_k^p \eta_k^{2q} \lambda_k^{-p}\right] \leq  \sum_{k=1}^{d^*} \left(\E\left[\eta_k^{2} \lambda_k\right]\right)^{p} \left(\E\left[\eta_k^{2} \lambda_k^{-p/q}\right]\right)^{q}~,
\end{align*}
thanks to Hölder inequality w.r.t.\ the expectation. Then, applying once again the Hölder inequality for the sum, we have
     \begin{align*}
      \E \left[\|\eta_t\|^2\right] &\leq  \left(\sum_{k=1}^{d^*} \E\left[\eta_k^{2} \lambda_k\right]\right)^{p} \left(\sum_{k=1}^{d^*}\E\left[\eta_k^{2} \lambda_k^{-p/q}\right]\right)^{q} \\
      &=   \left(\E 2 L(\theta)\right)^{p} \left(\E \left[\langle  \eta_t,  \Sigma^{-p/q} \eta_t   \rangle \right]\right)^{q}~,
    \end{align*}
    and the claim \eqref{eq:hoelder} follows.
We now prove that $t \mapsto \E \left(\langle  \eta_t,  \Sigma^{-p/q} \eta_t  \rangle \right)$ is bounded. Indeed, we proceed as before with Itô formula to get
    \begin{align*}
     \frac{d}{dt}\E \left(\langle  \eta_t,  \Sigma^{-p/q} \eta_t   \rangle \right) &= 2 \E \langle  \frac{d\eta_t}{dt},  \Sigma^{-p/q} \eta_t   \rangle + \E \langle  \frac{d\eta_t}{dt},  \Sigma^{-p/q} \frac{d\eta_t}{dt}   \rangle \\
     &= - 2 \E \langle  -\Sigma \eta_t ,  \Sigma^{-p/q} \eta_t   \rangle + \E \langle \sqrt{\frac{\gamma}{n}}  \X^\top R_\x (\theta_t) d B_t,  \Sigma^{-p/q} \sqrt{\frac{\gamma}{n}}  \X^\top R_\x (\theta_t) d B_t  \rangle \\
     &= - 2  \E \langle \eta_t ,  \Sigma^{1-p/q} \eta_t   \rangle + \frac{\gamma}{n}\E \Tr\left((\X^TR)^T \Sigma^{-p/q} \X^TR \right) \\
     &\leq \frac{\gamma}{n}\E \Tr\left(\Sigma^{-p/q} \X^\top \left(\diag(r_{\mathsf{x}}(\theta))^2 - \frac{1}{n}r_{\mathsf{x}}(\theta_t)\, r_{\mathsf{x}}(\theta_t)^\top \right) \X \right)  \\ 
     &\leq \frac{\gamma}{n} \E \sum_{i = 1}^n \langle x_i, \Sigma^{-p/q} x_i \rangle (\langle \eta_t , x_i \rangle-y_i)^2  \\ 
     & \leq  2\gamma \K_{p/q} L(\theta_t)~, 
\end{align*}
where we have used that for all $i \in \llbracket 1,n \rrbracket$, we have $ \langle x_i, \Sigma^{-p/q} x_i \rangle \leq \K_{p/q}$. Then, by integrating with respect to $t$, it yields
     \begin{align*}
    \E \left(\langle  \eta_t,  \Sigma^{-p/q} \eta_t   \rangle \right) &\leq \langle  \eta_0,  \Sigma^{-p/q} \eta_0   \rangle +2 \gamma \K_{p/q} \int_0^t \E L(\theta_u) du  \\
    &\leq \langle  \eta_0,  \Sigma^{-p/q} \eta_0   \rangle + \frac{\gamma \K_{p/q}}{2-\gamma \K} \|\theta_0 - \theta_*\|^2~. \\
    \end{align*}

Hence, calling $\Cons =\frac{1}{2}( \langle \eta_0,  \Sigma^{-p/q} \eta_0   \rangle + \frac{\gamma \K_{p/q}}{2-\gamma \K} \|\theta_0 - \theta_*\|^2)^{-p/q}~,$ we have the inequality, for all $t \geq 0$,
\begin{align*}
    \E L(\theta_t) \geq \Cons \left(\E \|\eta_t\|^2\right)^{1/p}~,
\end{align*}
and this yields the inequality
\begin{equation}\label{eq:algr}
    \E \| \eta_t\|^2 \leq  \| \eta_0 \|^2  - \Cons \int_{0}^t \left(\E \| \eta_u\|^2\right)^{1/p} du~,
\end{equation}
that implies, from a slight modification of Gronwall Lemma that for all $t \geq 0$, we have
\begin{equation*}
\E \| \eta_t\|^2 \leq \left[ \frac{1}{\frac{1}{\| \eta_0\|^{2(1/p - 1)}} + (1/p - 1) \Cons t}\right]^{\frac{1}{1/p - 1}}~,
\end{equation*}
 this gives the result claim in the theorem. To see how the last inequality goes, we define $g(t)=\| \eta_0 \|^2  - \Cons \int_{0}^t \left(\E \| \eta_u\|^2\right)^{1/p} du$ (which is positive) and we rewrite \eqref{eq:algr} as
 \begin{equation*}
\left(\frac{g'(t)}{-\Cons}\right)^p \leq g(t) \Longleftrightarrow \frac{g'(t)}{g(t)^{1/p}}\geq -\Cons  \Longrightarrow \frac{1}{-1/p+1}(g(t)^{-1/p+1}-g(0)^{-1/p+1})\geq -\Cons t,
\end{equation*}
and thus
 \begin{equation*}
g(t) \leq \left[-\Cons t (-\frac{1}{p}+1)+\| \eta_0 \|^{2(-1/p+1)}\right]^{\frac{1}{-1/p+1}},
\end{equation*}
and we conclude with \eqref{eq:algr}. \\
\noindent To prove the convergence almost surely, we use the Itô Formula to obtain 
\begin{align*}
        \E \left(\| \theta_t - \theta_*\|^2 \big|\, \mathcal{F}_{s} \right) &= \| \theta_0 - \theta_*\|^2+  2 \int_{0}^{s}  \sqrt{\frac{\gamma}{n}} \langle  R_\x (\theta_u)^\top\X \left( \theta_u -  \theta_* \right),    d B_u \rangle  \\
        & \hspace{1cm}+ \E \big(\int_{0}^{t} \frac{\gamma}{n}\Tr\left(\X^\top \left(\diag(r_{\mathsf{x}}(\theta))^2 du -  \frac{1}{n}r_{\mathsf{x}}(\theta_u)\, r_{\mathsf{x}}(\theta_u)^\top \right) \X \right) \\
        & \hspace{1cm} -\frac{2}{n} \langle \left(\theta_u - \theta_* \right),  \X^\top \X \left(\theta_t - \theta_*\right)  \rangle du \big|\, \mathcal{F}_{s} \big),
    \end{align*}
    For $\gamma <  \frac{1}{2\K}$, we have proven in (i) that the term inside the conditional expectation value/integral is negative. We can thus overestimate the latter by integrating from $0$ to $s$ to obtain
    \begin{equation*}
     \E \left(\| \theta_t - \theta_*\|^2 \big|\, \mathcal{F}_{s} \right) \leq  \| \theta_s - \theta_*\|^2,
       \end{equation*}
       and we deduce that $\| \theta_t - \theta_*\|^2$ is a positive supermartingale, $L^1$ bounded and therefore converge almost surely to $0$.
 \end{proof}

\subsection{The noisy case}
\label{appsec:noisy_training}

We begin by showing why $V (\theta) := \frac{1}{2} \|\theta - \theta_*\|^2$ is a Lyapunov function for the dynamics. This is Lemma~\ref{lem:Lyapunov_noisy}. \\

\begin{proof}[\textbf{Lemma~\ref{lem:Lyapunov_noisy}}]
    First, as $L$ is a quadratic function, we have that for all $\theta \in \R^d$, we have $$  \langle \nabla L(\theta), \theta - \theta_* \rangle =2(L(\theta) - L(\theta_*)). $$
    We recall that $\X^\top \X \theta_*= \X^\top y$ and it follows that $\nabla L(\theta)=\Sigma(\theta-\theta_*)$. We thus deduce that for all $\theta \in \R^d$, 
    \begin{align*}
        \mathcal{L}V(\theta) &= - \langle \Sigma \left(\theta - \theta_*\right), \left(\theta - \theta_*\right) \rangle + \frac{\gamma}{2n} \Tr \left[ \X^\top R^2_{\x}(\theta) \X\right]  \\
        &= 2( L(\theta_*) - L(\theta) ) + \frac{\gamma}{2n} \Tr \left[ \X^\top R^2_{\x}(\theta) \X\right]  \\
        &\leq 2( L(\theta_*) - L(\theta) )+ \frac{\gamma}{2n} \sum_{i = 1}^n (\langle x_i, \theta \rangle - y_i)^2 \|x_i\|^2 \\
        &\leq 2 \left(L(\theta_*) - L(\theta)(1-\frac{\gamma \K}{2})\right)~.
    \end{align*}
This Lyapunov identity implies the existence of a stationary distribution as explained in \cite[Theorem 3.7]{khasminskii2011stochastic}. 
This implies that for all $t \geq 0$, the dynamics does not explode.
\end{proof}

\subsubsection{Invariant measure and convergence} \label{appsubsec:Invariant measure}
We now turn into proving the main theorem of this part on the quantitative convergence to the stationary distribution. 
\\

\begin{proof}[\textbf{Theorem \ref{thm:convergence_noisy}}]
    $\bf (i)$ Wassertein contraction comes essentially from \textit{coupling arguments}. Let $\gamma \leq \K^{-1}$ and $\rho^1_0, \rho^2_0 \in \mathbb{P}_2(\R^d)$ two possible initial distributions. Then by \cite[Theorem 4.1]{villani2009optimal}, there exists a couple of random variables $\theta^1_0, \theta^2_0$ such that $W_2^2 (\rho^1_0, \rho^2_0) = \E \left[ \|\theta^1_0 - \theta^2_0\|^2 \right]$. Let $(\theta^1_t)_{t \geq 0}$ (resp. $(\theta^2_t)_{t \geq 0}$) be the solution of the SDE \eqref{eq:train_SDE}, sharing the same Brownian motion $(B_t)_{t \geq 0}$. Then, for all $t \geq 0$, the random variable $(\theta^1_t, \theta^2_t)$ is a coupling between $\rho^1_t$ and $\rho^2_t$, and hence 
    \begin{align*}
        W_2^2 ( \rho^2_t, \rho^1_t) \leq \E\left[ \| \theta^1_t - \theta^2_t \|^2 \right].
    \end{align*}
    Moreover, we denote by $\|.\|_{\mathrm{HS}}$ the Frobenius norm and by Itô formula, we have 
    \begin{align*}
        \frac{d}{dt}\E\left[ \| \theta^1_t - \theta^2_t \|^2 \right] &= - \frac{2}{n}\E\langle \theta^1_t - \theta^2_t, \X^\top (\X \theta^1_t - y) -  \X^\top (\X \theta^2_t - y)  \rangle + \frac{\gamma}{n} \E\| \X^\top R_{\x} (\theta^1_t) - \X^\top R_{\x} (\theta^2_t) \|^2_{\mathrm{HS}} \\
        &= - \frac{2}{n} \E\| \X(\theta^1_t - \theta^2_t) \|^2 + \frac{\gamma}{n} \E\| \X^\top (R_{\x} (\theta^1_t) - R_{\x} (\theta^2_t)) \|^2_{\mathrm{HS}} \\
        &\leq  - \frac{2}{n} \E\| \X(\theta^1_t - \theta^2_t) \|^2 \\
        &\hspace{1cm} + \frac{2\gamma}{n} \E\left(\| \X^\top \diag(\langle \theta^1_t - \theta^2_t, x_i \rangle_i ) \|^2_{\mathrm{HS}} + \frac{1}{n^2}\| \X^\top (\langle\theta^1_t - \theta^2_t,  x_i \rangle)_i \mathds{1}^\top  \|^2_{\mathrm{HS}} \right)~.
    \end{align*}
    Furthermore, we have for all $\theta^1, \theta^2 \in \R^d$ that
    \begin{align*}
        \| \X^\top \diag(\langle \theta^1 - \theta^2, x_i \rangle_i ) \|^2_{\mathrm{HS}} &= \Tr \left( \X \X^\top \diag(\langle \theta^1 - \theta^2, x_i \rangle^2_i )  \right) \\
        &= \sum_{i = 1}^n \|x_i\|^2 \langle \theta^1 - \theta^2, x_i \rangle^2  \\
        &\leq \K \|\X (\theta^1 - \theta^2)\|^2~, 
    \end{align*} 
    and 
    \begin{align*}
    \frac{1}{n^2}\| \X^\top (\langle\theta^1 - \theta^2, x_i \rangle)_i \mathds{1}^\top  \|^2_{\mathrm{HS}} &=  \frac{1}{n} \Tr \left( \X \X^\top (\langle\theta^1 - \theta^2, x_i \rangle)_i (\langle\theta^1 - \theta^2,  x_i \rangle)_i^\top  \right) \\
     &\leq  \frac{1}{n} \Tr \left( \X \X^\top \right)  \Tr \left( (\langle\theta^1 - \theta^2, x_i \rangle)_i (\langle\theta^1 - \theta^2,  x_i \rangle)_i^\top  \right) \\
    &\leq  \K \|\X (\theta^1 - \theta^2)\|^2~,
  \end{align*} 
where we use from the first to the second line the inequality $\Tr(AB) \leq \Tr(A) \Tr(B)$ for any $A$ and $B$ positive semi-definite.
Altogether, this gives the inequality:
   \begin{align}\label{eq:t1-t2}
        \frac{d}{dt}\E\left[ \| \theta^1_t - \theta^2_t \|^2 \right] 
        &\leq  - \frac{2 - 4\gamma \K }{n} \E\| \X(\theta^1_t - \theta^2_t) \|^2 \\
        &\leq  - 2\mu(1 - 2 \gamma  \K) \E \|\theta^1_t - \theta^2_t \|^2~,
    \end{align}
    and by Gronwall Lemma, denoting $c_\gamma = 2\mu( 1 - 2\gamma \K)$ this gives that 
    \begin{align*}
    W_2^2 ( \rho^1_t, \rho^2_t) \leq \E\left[ \| \theta^1_t - \theta^2_t \|^2 \right] \leq e^{ - c_\gamma t} \E\left[ \| \theta^1_0 - \theta^2_0 \|^2 \right] = e^{ - c_\gamma t} W_2^2 (\rho^2_0, \rho^1_0)~.
    \end{align*}
    Now, for all $s \geq 0$ setting $\rho^1_0 = \rho_0 \in \P_2(\R^d)$ and $\rho^2_0 = \rho_s \in \mathbb{P}_2(\R^d)$, we have for all $t \geq 0$,
    \begin{align*}
    W_2^2 ( \rho_t, \rho_{t+s}) \leq  e^{ - c_\gamma t} W_2^2 (\rho_0, \rho_s)~,
    \end{align*}
    which shows that the process $(\rho_t)_{t \geq 0}$ is of Cauchy type, and since $(\mathbb{P}_2(\R^d), W_2)$ is a Polish space, $\rho_t \to \rho_* \in \mathbb{P}_2(\R^d) $  as $t$ grows to infinity.
Now, since there exists a stationary solution to the process, let us fix $ \rho^1_0 = \rho^* \in \mathbb{P}_2(\R^d)$ and $\rho^2_0 = \rho_0\in \mathbb{P}_2(\R^d)$. We have then, 
$$ W_2^2 ( \rho_t, \rho^*) \leq e^{ - c_\gamma t} W_2^2 (\rho_0, \rho^*)~, $$
which concludes the first part of the Theorem. 

\vspace{0.35cm}

$\bf (ii)$   We will use the same steps as for the proof of Theorem \ref{Th:convergConst} (ii). Again, one steadily checks that for $p,q \in (0,1), \, p+q=1$, we have, for all $t \geq 0$, 
        \begin{equation}
            \frac{1}{n}\E\left[ \| \X \left(\theta^1_t - \theta^2_t\right) \|^2 \right] \geq \frac{  \E\left[ \| \left(\theta^1_t - \theta^2_t\right) \|^2 \right]^{1/p}}{\E\left[ \langle \theta^1_t - \theta^2_t, \Sigma^{-p/q} (\theta^1_t - \theta^2_t) \rangle \right]^{q/p} }~.
        \end{equation}
By Ito formula, skipping the details, we get
\begin{align*}
     \frac{d}{dt}&\E \left(\langle  \theta^1_t - \theta^2_t,  \Sigma^{-p/q} (\theta^1_t - \theta^2_t )  \rangle \right) \\
     &= - 2  \E \langle \theta^1_t - \theta^2_t ,  \Sigma^{1-p/q} (\theta^1_t - \theta^2_t)   \rangle +\frac{\gamma}{n}\E \Tr\left( \Sigma^{-p/q} \X^T (R_{\x}(\theta_1)-R_{\x}(\theta_2))(R_{\x}(\theta_1)-R_{\x}(\theta_2))^\top \X
     \right) \\
     & \leq \frac{\gamma}{n}\E \Tr\left(  \Sigma^{-p/q} \X^T (R_{\x}(\theta_1)-R_{\x}(\theta_2))(R_{\x}(\theta_1)-R_{\x}(\theta_2))^\top \X 
     \right)  \\ 
     &= \frac{\gamma}{n} \|  (R_{\x}(\theta_1)-R_{\x}(\theta_2))^\top \X \Sigma^{-p/2q}\|^2_{\mathrm{HS}} \\
     & \leq  \frac{2\gamma}{n} \E\big(\| \Sigma^{-p/2q} \X^\top \diag(\langle \theta^1_t - \theta^2_t, x_i \rangle_i ) \|^2_{\mathrm{HS}} + \frac{1}{n^2}\| \Sigma^{-p/2q} \X^\top (\langle\theta^1_t - \theta^2_t,  x_i \rangle)_i \mathds{1}^\top  \|^2_{\mathrm{HS}} \big)\\
     &= \frac{2\gamma}{n} \E \big[\Tr \left( \X \Sigma^{-p/q} \X^\top \diag(\langle \theta^1 - \theta^2, x_i \rangle^2_i )  \right) + \frac{1}{n} \Tr\left(  \X\Sigma^{-p/q} \X^\top (\langle\theta^1_t - \theta^2_t,  x_i \rangle)_i (\langle\theta^1_t - \theta^2_t,  x_i \rangle)_i^{\top}  \right) \big]\\
     & \leq \frac{2\gamma}{n}\E \big[ \sum_{i = 1}^n \langle x_i, \Sigma^{-p/q} x_i \rangle \langle \theta^1 - \theta^2, x_i \rangle^2+\frac{1}{n}\Tr\left(  \X\Sigma^{-p/q} \X^\top\right) \Tr\left( (\langle\theta^1_t - \theta^2_t,  x_i \rangle)_i (\langle\theta^1_t - \theta^2_t,  x_i \rangle)_i^{\top}  \right)\big]~,
\end{align*}
the last line by the fact that $\Tr(AB) \leq \Tr(A) \Tr(B)$ for $A$ and $B$ positive semi-definite. We thus conclude that
\begin{align*}
     \frac{d}{dt}&\E \left(\langle  \theta^1_t - \theta^2_t,  \Sigma^{-p/q}( \theta^1_t - \theta^2_t )  \rangle \right) \leq \frac{4 \gamma}{n} \E \|\X (\theta^1 - \theta^2)\|^2  \K_{p/q},
\end{align*}  
where we have used that for all $i \in \llbracket 1,n \rrbracket$, we have $ \langle x_i, \Sigma^{-p/q} x_i \rangle \leq \K_{p/q}$.
In addition, by \eqref{eq:t1-t2}, we get   
\begin{equation*}
\int_{0}^{t}\E\left[ \| \X(\theta^1_u - \theta^2_u) \|^2 \right] du  \leq   \frac{n}{2 - 4\gamma \K} \E\| (\theta^1_0 - \theta^2_0) \|^2~.
        \end{equation*}
        Collecting the estimates, we thus get
        \begin{equation*}
           \frac{1}{n}\E\left[ \| \X \left(\theta^1_t - \theta^2_t\right) \|^2 \right] \geq C  \E\left[ \| \left(\theta^1_t - \theta^2_t\right) \|^2 \right]^{1/p},
        \end{equation*}
        with $C=\left(\E\left[ \langle \theta^1_0 - \theta^2_0, \Sigma^{-p/q} (\theta^1_0 - \theta^2_0) \rangle \right]+2 \gamma \frac{\K_{p/q}}{1-2\gamma\K}\E(\|\theta^1_0-\theta^2_0\|^2)\right)^{-q/p}$,
        which combined with \eqref{eq:t1-t2}, gives
        \begin{align*}
        \E\left[ \| \theta^1_t - \theta^2_t \|^2 \right] 
        &\leq \E\left[ \| \theta^1_0 - \theta^2_0 \|^2 \right] - C\left(2 - 4\gamma \K\right) \int_{0}^{t} \E\| (\theta^1_u - \theta^2_u) \|^2du~,
    \end{align*}
that implies, from a slight modification of Gronwall Lemma that for all $t \geq 0$ (for $\Cons=2C\left(1 - 2\gamma \K)\right)$), we have
\begin{equation*}
\E\left[ \| \theta^1_t - \theta^2_t \|^2 \right] \leq \left[ \E\left[ \| \theta^1_0 - \theta^2_0 \|^2 \right]^{1  - 1/p} + (1/p - 1) \Cons t\right]^{\frac{1}{1 - 1/p }}~,
\end{equation*}
and we conclude as for (i).
\end{proof}

\subsubsection{Localization of the invariant measure.}
\label{appsubsec:Localization}

We now turn in the localization of the invariant measure presented in Proposition~\ref{prop:localization}.
\\

\begin{proof}[\textbf{Proposition~\ref{prop:localization}}]
We proved that $\theta_t$ tends weakly to $\Theta_*$ for $\gamma \leq 1/\K$. To insure that the first and second moment of $\theta_t$ tends to the first and two moments of $\Theta_*$, we first prove that it exists a $M$ such that for all $t$,  $\E\left\|\theta_t - \theta_*\right\|^4 \leq M$. By the Itô formula, 
\begin{align*}
    \frac{d}{dt}\E [V(\theta)]=\frac{d}{dt} \frac{1}{2}\E [\| \theta_t - \theta_* \|^2] 
    &\leq \E\left(-\frac{1}{n} \|\X (\theta_t - \theta_*) \|^2+\frac{\gamma \K }{2n} \|\X \theta_t - y\|^2\right) \notag\\
    &\leq  \E\left((-\frac{1}{n}+\frac{\gamma \K }{n}) \|\X (\theta_t - \theta_*) \|^2+\frac{\gamma \K }{n} \|\X \theta_t^* - y\|^2\right) \notag\\
    &\leq(-2+2\gamma \K ) \mu V(\theta)+2 \gamma \K  \sigma^2,
\end{align*}
 the second line by triangle inequality and the third line because $L(\theta_*)=\sigma^2$.
This implies by Gronwall Lemma that
\begin{equation}\label{eq:V}
    \E [V(\theta)]\leq e^{-2(1-\gamma \K)t}\left(\frac{1}{2}\| \theta_0 - \theta_* \|^2-\frac{ \gamma \K  \sigma^2}{(1-\gamma \K ) \mu}\right)+\frac{ \gamma \K  \sigma^2}{(1-\gamma \K )\mu}~,
\end{equation}
thus  $t \mapsto \E [\| \theta_t \|^2]$ is uniformly upper bounded for $\gamma<1/\K$. Again, by Ito formula, we get
 \begin{align*}
    \frac{d}{dt} \E\left\|\theta_t - \theta_*\right\|^4 
    &= -4 \E \left\| \theta_t - \theta_* \right\|^2 \langle \theta_t-\theta_*, \Sigma(\theta-\theta*) \rangle\\
    & \hspace{1cm} +\frac{\gamma}{2n} \E \Tr\left(\X^\top R_\x (\theta_t) R_\x (\theta_t)^{\top}\X \left(8 (\theta_t-\theta_*)(\theta_t-\theta_*)^{\top}+4  \left\|\theta_t - \theta_*\right\|^2I_d\right)\right) \\
    &\leq -\frac{4}{n}  \E \left\| \theta_t - \theta_* \right\|^2 \left\| \X (\theta_t - \theta_*) \right\|^2\\
    & \hspace{1cm}+\frac{\gamma}{2n} \E \Tr\left(\X^\top \diag(r_{\mathsf{x}}(\theta))^2\X \left(8 (\theta_t-\theta_*)(\theta_t-\theta_*)^{\top}+4  \left\|\theta_t - \theta_*\right\|^2I_d\right)\right) \\
    &\leq -\frac{4}{n}  \E \left\| \theta_t - \theta_* \right\|^2 \left\| \X (\theta_t - \theta_*) \right\|^2\\
    & +\frac{8\gamma}{2n} \E \Tr\left(\X^\top \diag(r_{\mathsf{x}}(\theta))^2\X \left( (\theta_t-\theta_*)(\theta_t-\theta_*)^{\top}\right)\right) +\frac{4\gamma \K}{2n} \E \left\|\theta_t - \theta_*\right\|^2 \left\|\X\theta_t - y\right\|^2\\
    &\leq -\frac{4}{n}  \E \left\| \theta_t - \theta_* \right\|^2 \left\| \X (\theta_t - \theta_*) \right\|^2+\frac{6 \gamma \K}{n}\E\left\|\X\theta_t - y\right\|^2\left\|\theta_t - \theta_*\right\|^2, 
\end{align*}
the last line by the fact that $\Tr(AB) \leq \Tr(A) \Tr(B)$ for $A$ and $B$ positive semi-definite and we deduce with the triangle inequality that the latter is less than
\begin{align*}
       \frac{d}{dt} \E\left\|\theta_t - \theta_*\right\|^4 &\leq \left(-\frac{4}{n}+\frac{12 \gamma \K}{n}  \right)\E \left\| \theta_t - \theta_* \right\|^2 \left\| \X (\theta_t - \theta_*) \right\|^2+\frac{12 \gamma \K}{n}\E\left\|\X\theta_* - y\right\|^2\left\|\theta_t - \theta_*\right\|^2 \\
    &\leq (-4+12 \gamma \K  ) \mu \E \left\| \theta_t - \theta_* \right\|^4 +24 \gamma \K \sigma^2 \E\left\|\theta_t - \theta_*\right\|^2,
 \end{align*}
 the last term is bounded by \eqref{eq:V} and we conclude with Gronwall Lemma for $\gamma<\frac{1}{3 \K}$ the claim on the fourth moment.

    For the mean, we consider the equation satisfies by $\E(\theta_t)$: $
        \frac{d}{dt} \E(\theta_t) = - \nabla L (\E(\theta_t))~,$
    by linearity of $\theta \mapsto \nabla L(\theta)$. Hence, we have 
    \begin{align*}
        \frac{d}{dt} \frac{1}{2}  \left\|\E(\theta_t) - \theta_*\right\|^2 = - \langle \Sigma (\E(\theta_t) - \theta_*), \E(\theta_t) - \theta_* \rangle \leq -\mu  \left\|\E(\theta_t) - \theta_*\right\|^2~,
    \end{align*}
    which gives that $\E(\theta_t) \to \theta_*$ as $t$ goes to infinity. Combining the latter with the boundedness of the fourth moment (One actually needs only a $1+\epsilon$ moment) and the weak convergence of $\theta_t$ to $\Theta_*$ imply the first claim of the proposition. By \eqref{eq:V}, we recall that

\begin{align*}
 \E [V(\theta)]\leq e^{-2(1-\gamma \K)t}\left(\frac{1}{2}\E\| \theta_0 - \theta_* \|^2-\frac{ \gamma \K  \sigma^2}{(1-\gamma \K ) \mu}\right)+\frac{ \gamma \K  \sigma^2}{(1-\gamma \K )\mu},
\end{align*}
which which combined with the boundedness of the fourth moment implies by taking the limit in the inequality, and weak convergence of $\rho_t$ towards $\rho^*$ that $\theta_* \sim \rho^*$ satisfies the claimed inequality in the proposition. 
\end{proof}
We turn into proving the moment explosion of the invariant distribution. This corresponds to proving Proposition \ref{prop:momentexplo}. First, we need the following lemma that aims at lower bounding some quartic form of $\theta$: 
\begin{lemma} \label{lem:almost_def_pos}
    For $n\geq 2d$, there exists a constant $\mathsf{c}> 0$ that depends only on $\X, y$ such that  $ \forall \eta \in \R^d$, $$\langle \X\eta, R_\x^2  \X\eta \rangle \geq \mathsf{c} \|\eta\|^2\|\X\eta\|^2~.$$
\end{lemma}

We assume the validity of Lemma \ref{lem:almost_def_pos} for the time being and turn into the proof of Proposition \ref{prop:momentexplo}.

\begin{proof}{\textbf{of Proposition \ref{prop:momentexplo}}}
Using Itô's Lemma, we obtain that
\begin{align*}
    \frac {d}{dt} \, \E [ \frac{1}{2}\|\eta_t \|^{2\alpha}] &=-\alpha \E\|\eta_t\|^{2(\alpha-1)} \langle \Sigma\eta_t,\eta_t \rangle +\E \frac{\gamma \alpha(\alpha-1) \|\eta_t\|^{2(\alpha-2)}}{n} \Tr\left(\X^\top R_\x^2 \X \eta_t \eta_t^T \right) \\
    &+\E \frac{\gamma \alpha \|\eta_t\|^{2(\alpha-1)}}{2n} \Tr\left(\X^\top R_\x^2 \X \right).
\end{align*}
We suppose that $\|\Theta_*\|$ has moments of order $2 \alpha$. We take $\eta_0=\Theta_*$ and thus obtain
\begin{equation*}
  \E\|\Theta_*\|^{2(\alpha-1)} \langle \Sigma\Theta_*,\Theta_* \rangle \geq \frac{\gamma (\alpha-1)}{n} \E \|\Theta_*\|^{2(\alpha-2)}\Tr\left(\X^\top R_\x^2 \X \Theta_* \Theta_*^\top \right).
\end{equation*}
Using the Lemma \ref{lem:almost_def_pos}, we deduce that
\begin{equation*}
  \E\|\Theta_*\|^{2\alpha} \geq \mathsf{c} \frac{\mu}{\lambda} \gamma (\alpha-1) \E   \|\Theta_*\|^{2\alpha},
\end{equation*}
 which leads to a contradiction for $\alpha$ large enough and we thus deduce that $\|\Theta_*\|$ has no moments of order $2 \alpha$ for $\alpha > 1+\frac{\lambda}{\mu \mathsf{c} }$.
\end{proof}
 It remains to prove Lemma \ref{lem:almost_def_pos}. For the ease of writing and clarity we now denote the residual~$r = X\theta - y \in \R^n$, and the constant noise vector~$\sigma = \X \theta^* - y \in \R^n$, and note in this proof~$R^2_\x \equiv R^2_r$ and~$\X\eta = r - \sigma$ to emphasize that every thing here can be expressed as a function of the residuals~$r$.

    \begin{proof}{\textbf{of Lemma~\ref{lem:almost_def_pos}}}. We first give the kernel of~$R^2$ as a function of~$r$. Let $I = \{ i \in \llbracket 1 , n \rrbracket, \ r_i = 0 \}$ and $\alpha_r=(1/r_1, 1/r_2, \hdots, 1/r_n)^\top$, then 
        %

           \begin{enumerate}[label=(\roman*),ref=\roman*]
                \item \label{i} \text{ if } $I \neq \emptyset,  \ker(R^2_r) = \mathrm{span} (e_i)_{i \in I}$ 
                \item \label{ii} \label{eq:kernel_of_R} \text{ if } $I = \emptyset,   \ker(R^2_r)=\mathrm{span} (\alpha_r)$~.
            \end{enumerate}

        We fix $r \in \R^n$ and let $z \in \ker(R^2_r) $, that is 
        \begin{align*}
            \left(\diag(r)^2 - \frac{1}{n}r r^\top\right) z &= 0~,
        \end{align*}
        i.e. for all $i \in \llbracket 1 , n \rrbracket$, 
        \begin{align*}
            r_i^2 z_i = \frac{\langle r, z \rangle}{n} r_i~.
        \end{align*}
        If $I = \emptyset$, we have the relationship for all $i,j \in \llbracket 1 , n \rrbracket$, $r_i z_i = \frac{\langle r, z \rangle}{n} = r_j z_j$, which instantly gives the result.
        Otherwise if $I \neq \emptyset$, for all $i \in I^c$, we have $r_i z_i = \frac{\langle r, z \rangle}{n}$ and summing all these equality of $i \in I^c$, we have $$ \langle r, z \rangle = \sum_{i \in I^c} r_i z_i = (n - |I|) \frac{\langle r, z \rangle}{n}~, $$
        which implies, as $|I| \neq n$, that $z \perp r$. Recalling that for $i \in \I^c$, we have $z_i = \frac{\langle r, z \rangle}{r_i n} = 0$ and \eqref{i},\eqref{ii} follow.\\
        Let us remark that $r \mapsto R^2_r$ carries the homogeneity property that for all $r \in \R^n$, $$R^2_{r/\|r\|} = R^2_r/\|r\|^2~, $$ so that denoting the sphere $\mathcal{S}_{n-1} := \{ v \in \R^n, \text{ s.t. } \|v\| = 1 \}$, we have for  $u = r/\|r\| \in \mathcal{S}_{n-1}$ and $ v = (r - \sigma) / \|r - \sigma\| \in \mathcal{S}_{n-1} \cap \mathrm{Ran}(\X)$ (note that $v$ is defined i.f.f. $\eta$  is not equal to $0$), we have
        \begin{align}
        \label{eq:quadratic_sphere}
           \frac{\langle \X\eta, R_\x^2  \X\eta \rangle}{ \|\X\eta\|^2 \|\X\eta + \sigma\|^2} =  \langle v, R^2_u v \rangle~.
        \end{align}
        Now we also need the following intermediate result:
            Let $\Pi_u$ be the orthogonal projector into $\ker(R^2_u)^\perp$, we have that there exists a constant $c > 0$ that depends only on the data such that for all $\eta \in \R^d$,
            \begin{equation*}
                \|\Pi_u v\| \geq c~.
            \end{equation*}
There are two cases given by~\eqref{i} and \eqref{ii}, that allow to compute explicitly $\Pi_u$ for all $u \in \R^n$:

            \eqref{i} If $I \neq \emptyset$, we have $$ \|\Pi_u v\|^2 = \frac{\left\langle \sum_{i \in I^c} x_i x_i^\top \eta, \eta \right\rangle }{ \left\langle \sum^n_{i= 1} x_i x_i^\top \eta, \eta \right\rangle} \geq c~,$$
            because any covariance matrix of the form $\sum_{i \in I^c} x_i x_i^\top$ is invertible if $|I^c|\geq d$, which is verified for $n\geq 2d$.
            
             \eqref{ii} If $I = \emptyset$, 
             $$\|\Pi_u v\|^2 = 1 - \frac{\langle v, \alpha_r \rangle^2}{ \|\alpha_r\|^2} = 1 - \varphi(\eta)~,$$ where $$  \varphi(\eta)= \frac{\left(\sum_{i = 1}^n \frac{\langle x_i, \eta \rangle}{(\langle x_i, \eta \rangle - \sigma_i)} \right)^2}{\left(\sum_{i = 1}^n \langle x_i, \eta \rangle^2 \right) \left(\sum_{i = 1}^n (\langle x_i, \eta \rangle - \sigma_i)^{-2}\right)}. $$
             From the Cauchy-Schwarz inequality, $\varphi \in [0,1]$ and let $m = \sup_{\eta \in \R^d} \varphi(\eta)$. We show below that $ m < 1$.
             
             First, let's study the situation \textit{at infinity}. For this, we write $\eta = \bar{\eta} \|\eta\|$, where $\bar{\eta} = \eta / \|\eta\| \in \mathcal{S}_{d-1}$ and call $J = \{i \in \llbracket 1, n\rrbracket, \text{ such that } \langle x_i, \bar{\eta} \rangle = 0\}$.  We know that $| J | \leq d$. We fix $\bar{\eta}$, and study the limit of $\varphi$ as $\|\eta\|$ grows. 
             
             If $J \neq \emptyset$, we have 
             \begin{align*}
             \varphi (\eta) \sim_{\|\eta\| \to \infty} \frac{(n - |J|)^2}{\left(\sum_{i \in J^c}^n \langle x_i, \eta \rangle^2 \right) \left(\sum_{i \in J} \sigma_i^{-2} \right)}.
             \end{align*}
             Hence, if $J \neq \emptyset$, then $\varphi(\eta) \to 0$ when $\|\eta\| \to \infty$. Now if $J = \emptyset$, 
             \begin{align*}
             \varphi (\eta) \sim_{\|\eta\| \to \infty} \frac{n^2}{\left(\sum_{i = 1}^n \langle x_i, \bar{\eta} \rangle^2 \right) \left(\sum_{i = 1}^n \langle x_i, \bar{\eta} \rangle^{-2} \right)} := \phi(\bar{\eta}).
             \end{align*}
             Moreover, $M := \sup_{\bar{\eta} \in \mathcal{S}_{d-1}}\phi(\bar{\eta}) < 1$, because if $M=1$, it would be attained for a certain $\bar{\eta}_*$ by compactness of $\mathcal{S}_{d-1}$ and continuity of $\phi$, and if $M = 1$, then  $\phi(\bar{\eta}_*) = 1$ and by the equality case of Cauchy-Schwarz, this corresponds to the fact that there exists $\lambda \in \R$ such that $|\X \bar{\eta}_* | = \lambda$, which has no solution for $n>d$ (recall that because $\bar{\eta}_* \in \mathcal{S}_{d-1}$, $\lambda = 0$ offer no solution either). We conclude from this that $M < 1$. 
             
             Now either the supremum of $\varphi$ is $M$, and in this case we have $m = M < 1$ and we are done with the proof. Either, $ m > M$, and in this case, the supremum is attained in a compact set of $\R^d$, hence attained for a certain $\eta_* \in \R^d$ by continuity of $\varphi$. As previously, if $m = 1$, this corresponds to the equality case of Cauchy-Schwarz and hence the existence of $\lambda \in \R$ such that 
             $$ (\langle x_i, \eta \rangle  - \sigma_i)^2 \langle x_i, \eta \rangle^2 = \lambda~, $$
             which has no solution generically for $n > d$. Finally $m  = 1$ is impossible and we have proven that $ m < 1$, which corresponds to $c = 1 - m > 0$.

        Now if we take back Eq.\eqref{eq:quadratic_sphere}, and remark that $$\langle v, R^2_u v \rangle = \langle \Pi_u v, R^2_u \Pi_u v \rangle = \|\Pi_u\|^2 \langle \bar{\Pi}_u v, R^2_u \bar{\Pi}_u v \rangle \geq c\, \langle \bar{\Pi}_u v, R^2_u \bar{\Pi}_u v \rangle \geq c \, \lambda_{\min}(R^2_{u_{\vert \ker(R^2_u)^\perp}} ). $$ 
        Now define $\Xi : \mathcal{S}_{n-1} \to \R_{\geq 0}$ as $\Xi(u) = \lambda_{\min}(R^2_{u_{\vert \ker(R^2_u)^\perp}})$. We look for $\ell := \min_{u \in \mathcal{S}_{n-1}} \Xi(u)$. As $\Xi$ is a continuous function defined on a compact, it attains is minimum within this compact. If $\ell = 0$, then this means that  there exists $u^* \in \mathcal{S}_{n-1}$ such that $\Xi(u^*) = \lambda_{\min}(R^2_{u^*_{\vert \ker(R^2_{u^*})^\perp}}) = 0$, which is impossible by the fact that we have restricted $R^2_u$ orthogonalilly to its kernel. Hence $\ell > 0$ and finally we have  $$ \frac{\langle \X\eta, R_\x^2  \X\eta \rangle}{ \|\X\eta\|^2 \|\X\eta + \sigma\|^2} = \langle  v, R^2_u v \rangle \geq c \ell > 0~,$$
         which leads to the proof of the lemma considering that there exists $K> 0$ such that $\|\X\eta + \sigma\|^2 \geq K \|\eta\|^2$  and noting $\mathsf{c} = c \ell K$. Indeed, combining the fact that $\langle \Sigma\eta_t,\eta_t \rangle \leq 4(L(\theta_t)+L(\theta_*))$ by the triangle inequality and the fact that $L(\theta^*) \leq L(\theta_t)$, we obtain that  $\langle \Sigma\eta_t,\eta_t \rangle \leq 8L(\theta_t)$, which implies that
\[
\frac{\mu n}{4}\langle \eta_t,\eta_t \rangle \leq \|\X\eta + \sigma\|^2.
\]           
        \end{proof}

\subsubsection{Convergence of variance reduction techniques} \label{appsubsec:Reductiontechniques}
Now, we turn into the convergence of the time-average of the iterates. This corresponds to the proof of Proposition~\ref{prop:ergodic}.

\begin{proof}[\textbf{Proposition~\ref{prop:ergodic}}]
The idea we use to show the quantitative convergence of the average iterates comes from~\cite{mattingly2010convergence} and the use of Poisson equation. Indeed, let us define the map $\varphi : \R^d \to \R^d$, $\varphi(\theta) = \Sigma^{-1} \theta$, and note $(\varphi_i(\theta))_{i \in \llbracket 1, d \rrbracket}$ its coordinates. Note that, for all $ i \in \llbracket 1, d \rrbracket$, we have $\nabla \varphi_i(\theta) = \Sigma^{-1} e_i$, where $(e_i)_{i \in \llbracket 1, d \rrbracket}$ is the canonical basis of $\R^d$. Hence, we have the Poisson equation, for all $ i \in \llbracket 1, d \rrbracket$ and $\theta \in \R^d$,
\begin{align*}
    \mathcal{L} \varphi_i (\theta) = - \langle \Sigma (\theta - \theta_*), \Sigma^{-1} e_i \rangle = \theta_i - \theta_{*, i}~. 
\end{align*}
Hence considering the action of $\mathcal{L}$ on fields of $\R^d$ applied coordinate-wise, we have more generally $ \mathcal{L} \varphi (\theta) = \theta - \theta_* $. Thus, we have by Itô calculus that for all $s \geq 0$
\begin{align*}
    d\varphi (\theta_s) = (\theta_s - \theta_*) ds + \sqrt{\frac{\gamma}{n}} \Sigma^{-1}\X^\top R_\x(\theta_s) d B_s
\end{align*}
and thus integrating from $0$ to $t$ and dividing by $t$, and defining the martingale  $M_t =  -\frac{1}{\sqrt{n}} \X^\top \int_0^t R_\x(\theta_s) d B_s $ we have, 
\begin{align*}
     \bar{\theta}_t - \theta_* &= \frac{1}{t} \left(\varphi (\theta_t) -  \varphi (\theta_0)\right) + \frac{\sqrt{\gamma}}{t} \Sigma^{-1} M_t = \frac{\Sigma^{-1} (\theta_t - \theta_0)}{t} + \frac{\sqrt{\gamma}}{ t} \Sigma^{-1} M_t~.
\end{align*}
Hence, multiplying by $\Sigma$, taking norms and then expectations, we have 
\begin{align*}
     \E \| \Sigma (\bar{\theta}_t - \theta_*) \|^2 &\leq \frac{2 \E \| \theta_t - \theta_0 \|^2}{t^2} + \frac{2 \gamma}{ t^2} \E\|M_t\|^2~.
\end{align*}
Then, by Itô isometry, we have 
\begin{align*}
    \E\|M_t\|^2 &= \frac{1}{n} \E \left[ \int_0^t \| \X^\top  R_\x(\theta_s) \|^2_{\mathrm{HS}} ds \right] \\
    & = \frac{1}{n} \E \left[ \int_0^t \Tr \left( \X \X^\top  R^2_\x(\theta_s) \right) ds \right] \\
    & \leq 2\K  \int_0^t \E \left[L(\theta_s) \right]  ds~.
\end{align*}
Moreover, we have seen that in Lemma~\ref{lem:Lyapunov_noisy} that for $\gamma \leq \K^{-1}$, we have 
    \begin{align*}
        \frac{d}{dt} \frac{1}{2} \E \|\theta_t - \theta_*\|^2 &\leq  - \E L(\theta_t) + 2 \sigma^2~,
    \end{align*}
and hence we have the upper bound on the loss $\int_0^t \E \left[L(\theta_s) \right]  ds \leq \frac{1}{2} \|\theta_0 - \theta_*\|^2 + 2 \sigma^2 t $. Finally, this gives that 
\begin{align*}
    \E\|M_t\|^2 &\leq  \K \|\theta_0 - \theta_*\|^2 + 4  \K 
 \sigma^2 t ~.
\end{align*}
On the other hand, we have for all $t \geq 0$

\begin{align*}
     \E \|\theta_t - \theta_0\|^2 & \leq 2\E \|\theta_t - \theta_*\|^2+2\E \|\theta_* - \theta_0\|^2 \\
     &\leq 4e^{-2(1-\gamma \K)t}\left(\frac{1}{2}\| \theta_0 - \theta_* \|^2-\frac{ \gamma \K  \sigma^2}{(1-\gamma \K ) \mu}\right)+\frac{4 \gamma \K  \sigma^2}{(1-\gamma \K )\mu}+ 2 \|\theta_* - \theta_0\|^2\\
     &\leq 4\| \theta_0 - \theta_* \|^2,
\end{align*}
the second line by \eqref{eq:V} for $\gamma<1/\K$. Overall, we have the bound,
\begin{align*}
    \E \| \Sigma (\bar{\theta}_t - \theta_*) \|^2  &\leq  \frac{8 \gamma \K  \sigma^2 }{t}  + \frac{10 \| \theta_0 - \theta_* \|^2}{t^2} ~.
\end{align*}
\end{proof}

Now, we turn into the convergence of $\theta_t$ in the case of the step-size decay. This corresponds to the proof of Proposition~\ref{prop:stepsizedecay}.

\begin{proof}[\textbf{Proposition~\ref{prop:stepsizedecay}}]
By Ito formula, we obtain that
    \begin{align*}
        \frac{d}{dt} \E [\| \theta_t - \theta_* \|^2] 
        &\leq \E\left(-\frac{2}{n} \|\X (\theta_t - \theta_*) \|^2+\frac{\gamma_t \K }{n} \|\X \theta_t - y\|^2\right) \notag\\
        &\leq  \E\left((-\frac{2}{n}+\frac{2\gamma_t \K }{n}) \|\X (\theta_t - \theta_*) \|^2+\frac{2\gamma_t \K }{n} \|\X \theta_t^* - y\|^2\right) \notag\\
        &\leq(-2+2\gamma_t \K ) \mu \E [\| \theta_t - \theta_* \|^2] +4 \gamma_t \K  \sigma^2\\
        &\leq  -\mu \E [\| \theta_t - \theta_* \|^2] +4 \gamma_t \K  \sigma^2~.
\end{align*}
Applying the Gronwall lemma, we get
\begin{align*}
    \E [\| \theta_t - \theta_* \|^2] &\leq \E [\| \theta_0 - \theta_* \|^2] e^{-\mu t}+\int_{0}^{t}e^{-\mu(t-s)} 4\K  \sigma^2 /(2\K+s^{\alpha})  ds \\
    &\leq \E [\| \theta_0 - \theta_* \|^2] e^{-\mu t}+ e^{-\mu t/2} \int_{0}^{t/2} 4\K  \sigma^2 /(2\K)  ds+\int_{t/2}^{t} 4\K  \sigma^2 /s^{\alpha}  ds \\
     &\leq \left(\E [\| \theta_0 - \theta_* \|^2] +  t   \sigma^2\right) e^{-\mu t/2}+4\K \sigma^2 \frac{(t/2)^{1-\alpha}}{\alpha-1} \\
     &\leq \frac{1}{t^{\alpha-1}}\left(e^{-\alpha}\left(\E [\| \theta_0 - \theta_* \|^2]e(2(\alpha-1)/\mu)^{\alpha-1} +(2\alpha/\mu)^{\alpha}   \sigma^2\right) + \frac{2^{1+\alpha}\K \sigma^2}{(\alpha-1)} \right).
\end{align*}
\end{proof}

\section{Proofs of Section~\ref{SGD_online} : online SGD}
\label{AppendixB}

\subsection{The noiseless case}
\label{appsec:noiseless_pop}

We begin this proof's section with the proof of convergence in the noiseless case. This corresponds to Theorem~\ref{Th:convergConstpop}.

\begin{proof}[\textbf{Theorem~\ref{Th:convergConstpop}}]
\noindent $\bf (i)$ The key ingredient of the proof is the Gronwall Lemma. Combining the Itô formula with \eqref{eq:train_SDE} gives us
    \begin{align*}
       \frac{d}{dt} \E \| \theta_t - \theta_*\|^2 &=  -2 \E \langle \theta_t - \theta_*,  \Sigma (\theta_t - \theta^*)  \rangle + \gamma\E \Tr\left( \sigma (\theta_t) \sigma(\theta_t)^\top\right) \\
         &\leq  -2 \E \langle \theta_t - \theta_*,  \Sigma (\theta_t - \theta^*)  \rangle + \gamma \E \left[ \Tr\left( (\langle \theta_t, X \rangle - Y)^2 XX^\top\right)\right] \\
         &\leq  -2 \E \langle \theta_t - \theta_*,  \Sigma (\theta_t - \theta^*)  \rangle + \gamma \K \E \langle \theta_t - \theta_*,  \Sigma (\theta_t - \theta^*)  \rangle \\
         &\leq  -(2 - \gamma \K) \E \langle \theta_t - \theta_*,  \Sigma (\theta_t - \theta^*)  \rangle \\
         &\leq  -\mu (2 - \gamma \K) \E \| \theta_t - \theta_*\|^2~.
    \end{align*}
       By integrating the latter thanks to Gronwall Lemma, we get the result claim for (i). \\
    
\noindent $\bf (ii)$ We define $\eta_t:=\theta_t - \theta_*$ and recall that, thanks to the penultimate inequality of (i) we have by integration that:
    \begin{equation*}
        \E \| \eta_t\|^2 \leq  \| \theta_0 - \theta_*\|^2  - 2(2-\gamma \K)\int_{0}^t \E L(\theta_u)   du.
    \end{equation*}
    The first consequence of this inequality is that
    \begin{align*}
        \int_{0}^t \E L(\theta_u) du \leq  \frac{1}{2(2-\gamma \K)}\| \theta_0 - \theta_*\|^2~.
    \end{align*}
    
    We want to lower bound the term $\E L(\theta_u)$ without using the smallest eigenvalue of $\Sigma$, that is supposedly infinitely small. Similarly to what was done before, thanks to the Hölder inequality, if $p,q \in (0,1), \, p+q=1$, we have, for all $t \geq 0$, 
        \begin{equation}
            \left(\E \left[\|\eta_t\|^2\right]\right)^{\frac{1}{p}} \leq 2\E \left[ L(\theta_t)\right]  \left(\E \left[\langle  \eta_t,  \Sigma^{-p/q} \eta_t   \rangle \right]\right)^{q/p}~.
        \end{equation}
We now prove that $t \mapsto \E \left(\langle  \eta_t,  \Sigma^{-p/q} \eta_t  \rangle \right)$ is bounded. Indeed, we proceed as before with Itô formula to get
    \begin{align*}
     \frac{d}{dt}\E \left(\langle  \eta_t,  \Sigma^{-p/q} \eta_t   \rangle \right) 
     &= - 2  \E \langle \eta_t ,  \Sigma^{1-p/q} \eta_t   \rangle + \gamma\E \Tr\left(\sigma \sigma^\top \Sigma^{-p/q} \right) \\
     &\leq \gamma \E \Tr\left(\Sigma^{-p/q} \left( \langle \theta_t, X \rangle - Y \right)^2 X X^\top\right)  \\ 
     &\leq \gamma \E \left[\langle X, \Sigma^{-p/q} X \rangle (\langle \theta_t , X \rangle- Y)^2\right]  \\ 
     & \leq  2\gamma \K_{p/q} L(\theta_t)~, 
\end{align*}
where we have used that almost surely, we have $ \langle X, \Sigma^{-p/q} X \rangle \leq \K_{p/q}$. Then, by integrating with respect to $t$, it yields
     \begin{align*}
    \E \left(\langle  \eta_t,  \Sigma^{-p/q} \eta_t   \rangle \right) &\leq \langle  \eta_0,  \Sigma^{-p/q} \eta_0   \rangle +2 \gamma \K_{p/q} \int_0^t \E L(\theta_u) du  \\
    &\leq \langle  \eta_0,  \Sigma^{-p/q} \eta_0   \rangle + \frac{\gamma \K_{p/q}}{2-\gamma \K} \|\theta_0 - \theta_*\|^2~. \\
    \end{align*}

Hence, calling $\Cons =\frac{1}{2}( \langle \eta_0,  \Sigma^{-p/q} \eta_0   \rangle + \frac{\gamma \K_{p/q}}{2-\gamma \K} \|\theta_0 - \theta_*\|^2)^{-p/q}~,$ we have the inequality, for all $t \geq 0$,
\begin{align*}
    \E L(\theta_t) \geq \Cons \left(\E \|\eta_t\|^2\right)^{1/p}~,
\end{align*}
and this yields the inequality
\begin{equation}\label{eq:algro}
    \E \| \eta_t\|^2 \leq  \| \eta_0 \|^2  - \Cons \int_{0}^t \left(\E \| \eta_u\|^2\right)^{1/p} du~,
\end{equation}
that implies, from a slight modification of Gronwall Lemma that for all $t \geq 0$, we have
\begin{equation*}
\E \| \eta_t\|^2 \leq \left[ \frac{1}{\frac{1}{\| \eta_0\|^{2(1/p - 1)}} + (1/p - 1) \Cons t}\right]^{\frac{1}{1/p - 1}}~,
\end{equation*}
 this gives the result claim in the theorem. To see how the last inequality goes, we define $g(t)=\| \eta_0 \|^2  - \Cons \int_{0}^t \left(\E \| \eta_u\|^2\right)^{1/p} du$ (which is positive) and we rewrite \eqref{eq:algro} as
 \begin{equation*}
\left(\frac{g'(t)}{-\Cons}\right)^p \leq g(t) \Longleftrightarrow \frac{g'(t)}{g(t)^{1/p}}\geq -\Cons  \Longrightarrow \frac{1}{-1/p+1}(g(t)^{-1/p+1}-g(0)^{-1/p+1})\geq -\Cons t,
\end{equation*}
and thus
 \begin{equation*}
g(t) \leq \left[-\Cons t (-\frac{1}{p}+1)+\| \eta_0 \|^{2(-1/p+1)}\right]^{\frac{1}{-1/p+1}},
\end{equation*}
and we conclude with \eqref{eq:algro}. \\
\noindent To prove the convergence almost surely, we use the Itô Formula to obtain 
\begin{align*}
        \E \left(\| \theta_t - \theta_*\|^2 \big|\, \mathcal{F}_{s} \right) &= \| \theta_0 - \theta_*\|^2+  2 \sqrt{\gamma} \int_{0}^{s}   \langle \left( \theta_u -  \theta_* \right), \sigma(\theta_u)   d B_u \rangle  \\
        & \hspace{1cm}+ \gamma \E \left[\int_{0}^{t} \left[\Tr\left(\sigma \sigma^\top(\theta_u) \right) - 4 L(\theta_u) \right] du \big|\, \mathcal{F}_{s} \right],
    \end{align*}
    For $\gamma <  \frac{1}{2\K}$, we have proven in (i) that the term inside the conditional expectation value/integral is negative. We can thus overestimate the latter by integrating from $0$ to $s$ to obtain
    \begin{equation*}
     \E \left(\| \theta_t - \theta_*\|^2 \big|\, \mathcal{F}_{s} \right) \leq  \| \theta_s - \theta_*\|^2,
       \end{equation*}
       and we deduce that $\| \theta_t - \theta_*\|^2$ is a positive supermartingale, $L^1$ bounded and therefore converge almost surely to $0$.
 \end{proof}

\subsection{The noisy case}
\label{appsec:noise_lip}

We begin this section by proving that the multiplicative noise carries some Lipschitz property. This corresponds to Lemma~\ref{lem:noise_lip}.

\begin{proof}[\textbf{Lemma~\ref{lem:noise_lip}}]
Recall that the diffusion matrix writes
\begin{align*}
    \sigma(\theta) &= \left(f(\theta) - g(\theta)  \right)^{1/2}, \quad \text{with} \\
    f(\theta) = \E_\rho \left[ \left(\langle \theta, X \rangle - Y \right)^2 X X^\top \right] , &\text{ and } \ g(\theta) = E_\rho \left[ \left(\langle \theta, X \rangle - Y \right) X\right] E_\rho \left[ \left(\langle \theta, X \rangle - Y \right) X\right]^\top~.
\end{align*}
Let us introduce for all $\theta \in \R^d$, the differential operator of $\sigma$, that is $\mathsf{d} \sigma_\theta : \R^d \to \R^{d \times d}$ such that for all $h \in \R^d$, 
    $$ \sigma(\theta + h) = \sigma(\theta) +  \mathsf{d} \sigma_\theta(h) + o(\|h\|). $$ 
    We calculate, for all $\theta, h \in \R^d$,
    \begin{align*}
        \sigma(\theta + h) &= \left( f(\theta + h) - g(\theta + h) \right)^{1/2} \\
        &= \left( f(\theta) - g(\theta) + \mathsf{l}_{\theta} (h) + o(\|h\|) \right)^{1/2}
    \end{align*}
where $\mathsf{l}_{\theta} (h) = \mathsf{d}f_{\theta} (h) - \mathsf{d}g_{\theta} (h)$ and
\begin{align*}
    \mathsf{d}f_{\theta} (h) &= 2 E_\rho \left[ \left(\langle \theta, X  \rangle - Y\right) \langle h, X  \rangle XX^\top \right] \\
    \mathsf{d}g_{\theta} (h) &=\E\left[ \langle h, X \rangle X \right] E_\rho \left[ \left(\langle \theta, X \rangle - Y \right) X\right]^\top + E_\rho \left[ \left(\langle \theta, X \rangle - Y \right) X\right] \E\left[ \langle h, X \rangle X \right]^\top~.
\end{align*}
    Hence, introducing the square root operator $\psi : \mathbb{S}^d_{+} \to \mathbb{S}^d_{+}$, such that for all $M \in \mathbb{S}^d_{+}$, $\psi(M) = M^{1/2}$. We have, for all $\theta, h \in \R^d$, 
    \begin{align*}
        \mathsf{d} \sigma_\theta(h) =  \mathsf{d} \psi_{ \sigma^2(\theta) } \left( 
\mathsf{l}_{\theta} (h) \right), 
    \end{align*}
    that is that $\mathsf{d} \sigma_\theta(h)$ is the unique solution to the matrix Lyapunov equation
    \begin{align*}
        \sigma(\theta) \mathsf{d} \sigma_\theta(h) + \mathsf{d} \sigma_\theta(h) \sigma(\theta) = \mathsf{l}_{\theta} (h)~,
    \end{align*}
    or equivalently, expressed in close form as the expression
    \begin{align*}
        \mathsf{d} \sigma_\theta(h) = \int_{0}^{+ \infty} e^{- s \sigma(\theta)}  \mathsf{l}_{\theta} (h)  e^{- s \sigma(\theta)} ds~.
    \end{align*}
    That being written. Let us pose for $t \in [0,1]$, the function $\Psi(t) = \sigma(t \theta + (1-t) \eta)$, we have $$  \sigma(\theta) -  \sigma(\eta)   = \Psi(1) - \Psi(0)  =  \int_0^1 \Psi'(u) du  ~,$$
    and knowing that $\Psi'(u) = \mathsf{d}\sigma_{m_u} (\theta- \eta)$, where $m_u = u \theta + (1 - u) \eta$, and hence using the triangular inequality for the Hilbert-Schmidt norm, we have
    \begin{align*}
        \| \sigma(\theta) -  \sigma(\eta)   \|_{\mathrm{HS}} \leq \int_0^1 \| \mathsf{d}\sigma_{m_u} (\theta- \eta)\|_{\mathrm{HS}} du
    \end{align*}
    where $(m_u)_{ 0 \leq u \leq 1}$ is a parametrization of the line joining $\theta$ and $\eta$:  $m_u = u \theta + (1 - u) \eta$.

Hence, it remains to upper bound the Hilbert-Schmidt norm of the differential of $\sigma$ thanks to its integral representation presented above. In fact for all $\theta, h \in \R^d$, we have
    \begin{align*}
        \|\mathsf{d} \sigma_\theta(h)\|_{\mathrm{HS}} \leq \|\mathsf{l}_{\theta} (h)\| \int_{0}^{+ \infty} \|e^{- s \sigma(\theta)}\|^2_{\mathrm{HS}} ds~.
    \end{align*}
Moreover, we have $\|\mathsf{l}_{\theta} (h)\| \leq \|\mathsf{d}f_\theta(h)\| +  \|\mathsf{d}g_\theta(h)\|$ and using that $\|.\|_{\mathrm{HS}}$ is sub-multiplicative, we have
\begin{align*}
\|\mathsf{d}f_\theta(h)\| &\leq 2 \E\left[ \left|\langle \theta, X  \rangle - Y\right| |\langle h, X  \rangle| \|XX^\top\| \right] \\
&\leq 2 \K \sqrt{ \E\left[ \left|\langle \theta, X  \rangle - Y\right|^2\right]} \sqrt { \E \left[|\langle h, X  \rangle|^2 \right] } \\
&= 2 \sqrt{2} \K\sqrt{L(\theta)} \sqrt{\langle \Sigma h, h \rangle }~,
\end{align*}
and similarly, 
\begin{align*}
\|\mathsf{d}g_\theta(h)\| &\leq 2 \sqrt{2} \K \sqrt{L(\theta)} \sqrt{\langle \Sigma h, h \rangle }~.
\end{align*}
Finally, we have the bound
\begin{align*}
\int_{0}^{+ \infty} \|e^{- s \sigma(\theta)}\|^2_{\mathrm{HS}} ds = \int_{0}^{+ \infty} \mathrm{Tr}(e^{- 2s \sigma(\theta)}) ds = \frac{1}{2}{\mathrm{Tr} \left( \sigma^{-1}(\theta) \right)}~.
\end{align*}
%
As $\sigma^2(\theta)\succeq a^2 L(\theta) I_d$, we deduce that 
\begin{equation*}
L(\theta) \mathrm{Tr} \left( \sigma^{-1}(\theta) \right)^2 \leq   \mathrm{Tr} \left( L(\theta) \sigma^{-2}(\theta) \right) d \leq \frac{d^2}{a^2} ~,
\end{equation*}
the first inequality by Cauchy-Schwarz and the Lemma \ref{lem:noise_lip} follows.
\end{proof}

We can now turn into the proof of the main theorem of this section: the proof of convergence in the noisy case. This corresponds to proving Theorem~\ref{thm:convergence_noisy_pop}. \\

\begin{proof}[\textbf{Theorem~\ref{thm:convergence_noisy_pop}}]
    $\bf (i)$ Wassertein contraction comes essentially from \textit{coupling arguments}. Let $\gamma \leq \K^{-1}$ and $\rho^1_0, \rho^2_0 \in \mathbb{P}_2(\R^d)$ two possible initial distributions. Then by \cite[Theorem 4.1]{villani2009optimal}, there exists a couple of random variables $\theta^1_0, \theta^2_0$ such that $W_2^2 (\rho^1_0, \rho^2_0) = \E \left[ \|\theta^1_0 - \theta^2_0\|^2 \right]$. Let $(\theta^1_t)_{t \geq 0}$ (resp. $(\theta^2_t)_{t \geq 0}$) be the solution of the SDE \eqref{eq:train_SDE}, sharing the same Brownian motion $(B_t)_{t \geq 0}$. Then, for all $t \geq 0$, the random variable $(\theta^1_t, \theta^2_t)$ is a coupling between $\rho^1_t$ and $\rho^2_t$, and hence 
    \begin{align*}
        W_2^2 ( \rho^2_t, \rho^1_t) \leq \E\left[ \| \theta^1_t - \theta^2_t \|^2 \right].
    \end{align*}
    Moreover, we denote by $\|.\|_{\mathrm{HS}}$ the Frobenius norm and by Itô formula, we have 
    \begin{align*}
        \frac{d}{dt}\E\left[ \| \theta^1_t - \theta^2_t \|^2 \right] &= - 2\E \left[\langle \theta^1_t - \theta^2_t, \Sigma (\theta^1_t - \theta_*) -  \Sigma (\theta^2_t - \theta_*)  \rangle \right] + \gamma \E\| \sigma (\theta^1_t) - \sigma (\theta^2_t) \|^2_{\mathrm{HS}} \\
        &\leq - 2 \E \left[\langle \Sigma (\theta^1_t - \theta^2_t), \theta^1_t - \theta^2_t \rangle\right] + 2 \gamma \K c \E \left[\langle \Sigma (\theta^1_t - \theta^2_t), \theta^1_t - \theta^2_t \rangle\right]~,
     \end{align*}
     thanks to Lemma~\ref{lem:noise_lip}. Hence, this gives the inequality:
   \begin{align}\label{eq:toto}
        \frac{d}{dt}\E\left[ \| \theta^1_t - \theta^2_t \|^2 \right] 
        &\leq  - 2 (1 - \gamma \K c)\E \left[\langle \Sigma (\theta^1_t - \theta^2_t), \theta^1_t - \theta^2_t \rangle\right] \\
        &\leq  - 2\mu (1 - \gamma \K c)  \E \|\theta^1_t - \theta^2_t \|^2~,
    \end{align}
    and by Gronwall Lemma, denoting $c_\gamma = 2\mu( 1 - \gamma \K c)$ this gives that 
    \begin{align*}
    W_2^2 ( \rho^1_t, \rho^2_t) \leq \E\left[ \| \theta^1_t - \theta^2_t \|^2 \right] \leq e^{ - c_\gamma t} \E\left[ \| \theta^1_0 - \theta^2_0 \|^2 \right] = e^{ - c_\gamma t} W_2^2 (\rho^2_0, \rho^1_0)~.
    \end{align*}
    Now, for all $s \geq 0$ setting $\rho^1_0 = \rho_0 \in \P_2(\R^d)$ and $\rho^2_0 = \rho_s \in \mathbb{P}_2(\R^d)$, we have for all $t \geq 0$,
    \begin{align*}
    W_2^2 ( \rho_t, \rho_{t+s}) \leq  e^{ - c_\gamma t} W_2^2 (\rho_0, \rho_s)~,
    \end{align*}
    which shows that the process $(\rho_t)_{t \geq 0}$ is of Cauchy type, and since $(\mathbb{P}_2(\R^d), W_2)$ is a Polish space, $\rho_t \to \rho_* \in \mathbb{P}_2(\R^d) $  as $t$ grows to infinity.
Now, since there exists a stationary solution to the process, let us fix $ \rho^1_0 = \rho^* \in \mathbb{P}_2(\R^d)$ and $\rho^2_0 = \rho_0\in \mathbb{P}_2(\R^d)$. We have then, 
$$ W_2^2 ( \rho_t, \rho^*) \leq e^{ - c_\gamma t} W_2^2 (\rho_0, \rho^*)~, $$
which concludes the first part of the Theorem. 

\vspace{0.35cm}

$\bf (ii)$   We will use the same steps as for the proof of Theorem \ref{Th:convergConst} (ii). Again, one steadily checks that for $p,q \in (0,1), \, p+q=1$, we have, for all $t \geq 0$, 
        \begin{equation}
          \E\left[ \| \Sigma^{1/2} \left(\theta^1_t - \theta^2_t\right) \|^2 \right] \geq \frac{  \E\left[ \| \theta^1_t - \theta^2_t \|^2 \right]^{1/p}}{\E\left[ \langle \theta^1_t - \theta^2_t, \Sigma^{-p/q} (\theta^1_t - \theta^2_t) \rangle \right]^{q/p} }~.
        \end{equation}
By Ito formula, skipping the details, we get
\begin{align*}
     \frac{d}{dt}\E \left(\langle  \theta^1_t - \theta^2_t,  \Sigma^{-p/q} (\theta^1_t - \theta^2_t )  \rangle \right) & \leq \gamma \E \Tr\left(  \Sigma^{-p/q} (\sigma(\theta_1) - \sigma(\theta_2))^2  \right)  \\ 
     & \leq  2 \gamma c_{p/q} \K_{p/q} \langle \Sigma (\theta_1 - \theta_2), \theta_1 - \theta_2\rangle ~,
\end{align*}
thanks to assumption \eqref{eq:lip_chelou}. In addition, by \eqref{eq:toto}, we get   
\begin{equation*}
\int_{0}^{t}\E\left[ \| \Sigma^{1/2} \left(\theta^1_u - \theta^2_u\right) \|^2 \right] du  \leq   \frac{1}{2(1 - \gamma \K c)} \E\| \theta^1_0 - \theta^2_0 \|^2~.
        \end{equation*}
        Collecting the estimates, we thus get
 \begin{equation*}
           \E \left(\langle  \theta^1_t - \theta^2_t,  \Sigma^{-p/q} (\theta^1_t - \theta^2_t )  \rangle \right) \leq \frac{ \gamma c_{p/q} \K_{p/q}}{1 - \gamma \K c} \E\| \theta^1_0 - \theta^2_0 \|^2 +  \E \left(\langle  \theta^1_0 - \theta^2_0,  \Sigma^{-p/q} (\theta^1_0 - \theta^2_0 )  \rangle \right)~.
\end{equation*}
That is to say that
       \begin{equation}
          \E\left[ \| \Sigma^{1/2} \left(\theta^1_t - \theta^2_t\right) \|^2 \right] \geq C \E\left[ \| \theta^1_t - \theta^2_t \|^2 \right]^{1/p}~,
        \end{equation}
        with $C=\left(\E\left[ \langle \theta^1_0 - \theta^2_0, \Sigma^{-p/q} (\theta^1_0 - \theta^2_0) \rangle \right]+ \frac{\gamma c_{p/q}\K_{p/q}}{1-\gamma c\K}\E(\|\theta^1_0-\theta^2_0\|^2)\right)^{-q/p}$,
        which combined with equation \eqref{eq:toto}, gives
        \begin{align*}
        \E\left[ \| \theta^1_t - \theta^2_t \|^2 \right] 
        &\leq \E\left[ \| \theta^1_0 - \theta^2_0 \|^2 \right] - 2 C\left(1 - \gamma c \K\right) \int_{0}^{t} \E\left[\| (\theta^1_u - \theta^2_u) \|^2\right]^{1/p}du~,
    \end{align*}
that implies, from a slight modification of Gronwall Lemma that for all $t \geq 0$ (for $\Cons=2C\left(1 - \gamma c \K)\right)$), we have
\begin{equation*}
\E\left[ \| \theta^1_t - \theta^2_t \|^2 \right] \leq \left[ \E\left[ \| \theta^1_0 - \theta^2_0 \|^2 \right]^{1  - 1/p} + (1/p - 1) \Cons t\right]^{\frac{1}{1 - 1/p }}~,
\end{equation*}
and we conclude as for (i).
\end{proof}

\end{document}